\documentclass{article}

\usepackage{microtype}
\usepackage{graphicx}
\usepackage{subfigure}
\usepackage{booktabs} 

\usepackage{enumitem}

\usepackage[accepted]{icml2024}
\usepackage{algorithm}
\usepackage{algorithmic}

\usepackage{amsmath}
\usepackage{amssymb}
\usepackage{mathtools}
\usepackage{amsthm}
\usepackage{bbm}
\usepackage{mathrsfs}

\DeclareMathOperator{\logit}{logit}
\usepackage[rightcaption]{sidecap}
\usepackage{float}

\usepackage[capitalize,noabbrev]{cleveref}

\theoremstyle{plain}
\newtheorem{theorem}{Theorem}[section]

\newtheorem{lemma}[theorem]{Lemma}
\newtheorem{corollary}[theorem]{Corollary}
\theoremstyle{definition}

\theoremstyle{remark}

\newcommand{\EE}[0]{\mathbb{E}}

\DeclareMathOperator*{\argmin}{arg\,min}

\usepackage[textsize=tiny]{todonotes}

\icmltitlerunning{Deep LTMLE}

\begin{document}

\twocolumn[
\icmltitle{Longitudinal Targeted Minimum Loss-based Estimation with Temporal-Difference Heterogeneous Transformer}




\begin{icmlauthorlist}
\icmlauthor{Toru Shirakawa}{ou,ucb}
\icmlauthor{Yi Li}{ucb}
\icmlauthor{Yulun Wu}{ucb}
\icmlauthor{Sky Qiu}{ucb}
\icmlauthor{Yuxuan Li}{ucb} \\
\icmlauthor{Mingduo Zhao}{ucb}
\icmlauthor{Hiroyasu Iso}{ou,ncgm}
\icmlauthor{Mark van der Laan}{ucb}
\end{icmlauthorlist}

\icmlaffiliation{ou}{Osaka University Graduate School of Medicine, Suita, Japan}
\icmlaffiliation{ucb}{Univerity of California, Berkeley, United States}
\icmlaffiliation{ncgm}{National Center for Global Health and Medicine, Tokyo, Japan}

\icmlcorrespondingauthor{Toru Shirakawa}{shirakawa@pbhel.med.osaka-u.ac.jp}

\icmlkeywords{Machine Learning, ICML}

\vskip 0.3in
]



\printAffiliationsAndNotice{}  

\begin{abstract}
We propose Deep Longitudinal Targeted Minimum Loss-based Estimation (Deep LTMLE), a novel approach to estimate the mean of counterfactual outcome under dynamic treatment policies in longitudinal problem settings. Our approach utilizes a transformer architecture with heterogeneous type embedding trained using temporal-difference learning. After obtaining an initial estimate using the transformer, following the targeted minimum loss-based likelihood estimation (TMLE) framework, we statistically corrected for the bias commonly associated with machine learning algorithms. Furthermore, our method also facilitates statistical inference by enabling the provision of 95\% confidence intervals grounded in asymptotic statistical theory. Simulation results demonstrate our method's superior performance over existing approaches, particularly in complex, long time-horizon scenarios. It remains effective in small-sample, short-duration contexts, matching the performance of asymptotically efficient estimators. To demonstrate our method in practice, we applied our method to estimate counterfactual mean outcomes for standard versus intensive blood pressure management strategies in a real-world cardiovascular epidemiology cohort study.
\end{abstract}


\section{Introduction}
In the fields of medicine and public health, researchers frequently encounter data that are both high-dimensional and longitudinal. The outcomes of interest in these settings often involve time to the incidence of some failure event, such as total mortality \cite{van2012unified, salerno2023high}. Estimating the counterfactual probability of the event is challenging in high-dimensional longitudinal settings. Existing methods suffer computationally due to lack of scalability and have worse performance due to curse-of-dimensionality \cite{wyss_machine_learning_2022}. In response, we propose an estimator that is computationally scalable and simultaneously allows for robust statistical inference. Our estimator incorporates a transformer architecture for estimating the target estimand, defined as the cumulative incidence probability under dynamic interventions, where the treatment sequence depends on patients' evolving histories. The target estimand can be identified through the g-formula contingent upon suitable assumptions \cite{robins1986new}. However, the target functional involves integration over potentially high-dimensional time-dependent covariates across time-horizon, posing computational challenges. Our method advances the longitudinal targeted minimum loss-based estimation (LTMLE) framework \cite{laanTargetedMinimumLoss2012, ltmle_rpackage} by leveraging the computational capabilities of the transformer, facilitating the estimation of the target estimand and relevant nuisance parameters.


A number of estimators for the target estimand were proposed since the pioneering work by Robins \cite{robins1986new}. These estimators first factor the target parameter as a functional of nuisance parameters given a structural assumption on the underlying variables. Then, a common strategy to construct an estimator is plug-in, where one estimate the nuisance components with some models and then plug them into the target functional. However, since the naive plug-in of the estimated nuisance components causes bias, several methods have been proposed to remove this bias using the first variation of the target functional called influence function. Examples of such debiasing techniques include one-step estimators \cite{klaassen1987,bkrw}, estimating equations \cite{robins1994, cherLocally}, and targeted minimum loss-based estimation (TMLE) \cite{laanTLBook2011}. Notably, due to its plug-in property, TMLE stands out because it will respect any conditional bounds on the outcome or global bounds on the statistical model, resulting in improved finite-sample performance \cite{tmle_plug_in}. 

The first-order bias of the plug-in estimator is represented as a population mean of the influence function evaluated at the estimated nuisance distribution. Bias correction is performed by solving the empirical analogue of this term. TMLE solves this term by optimizing a loss function along a submodel starting from the initial nuisance estimate \cite{bangDoublyRobustEstimation2005, laanTMLE2006, laanTLBook2011}. The loss function and the submodel are chosen so that the linear span of the derivative of the loss function along the submodel contains the efficient influence function, the influence function with minimal variance. Targeting is the term that refers to this correction by fluctuating of the initial estimate along the path.

The current LTMLE, a TMLE developed in the context of longitudinal data, relies on a sequential regression representation of the target estimand \cite{bangDoublyRobustEstimation2005}. An ensemble machine learning technique called \textit{super learner} is then used to estimate the nuisance components of the data-generating distribution \cite{superlearner}. In real-world complex longitudinal data, these nuisance components, such as the survival probability at a given time, may depend on all past histories. Therefore, the Markovian property, which states that future variable values only depends on the present variables, independent of the past, is not guaranteed to hold. In other words, every observed variable could depend on the past variables in the time ordering. Hence, we want our model for the nuisance components to be able to take variable length of history as input. Furthermore, LTMLE performs sequential regressions starting from the last time point and thus cannot pool data across time points. Therefore, the error from the first regression at the last time point, which is the hardest regression because of the least observations due to failure and censoring with the most high-demensional covariates with all the history through the observation period, propagates to the final estimate. Under the targeted learning framework, we introduce a transformer architecture tailored towards our longitudinal setting capable of pooling time points, and propose a novel method for the bias correction using a single fluctuation parameter across all time-points.

Our contribution includes:
1) Developed a general method that uses a transformer architecture to facilitate valid statistical inference in longitudinal settings concerning survival outcomes under dynamic interventions;
2) Proposed a method for bias correction using one-dimensional fluctuation for any length of time-horizon;
3) Demonstrated competetive statistical performance with asymptotically efficient estimators in simple and low-dimensional settings and superior statistical and computational performances in more complex settings;
and 4) Applied our method to a real-world medical data with results presented in a format that aligns with clinical research recommendations \citep{hernan_target_2022, dahabreh_causal_2024}.

\section{Related Work}

In the data science literature, several methods were proposed that predict the counterfactual outcomes from patient history. The methods include G-Net \cite{liGNetDeepLearning2020a}, counterfactual recurrent network (CRN) \cite{Bica2020Estimating}, and causal transformer (CT) \cite{pmlr-v162-melnychuk22a}. However, their target parameters do not involve survival outcomes, and their methods are optimized for the mean squared error (MSE) of the individual predictions, rather than for making statistical inferences. DeepACE \cite{Frauen_Hatt_Melnychuk_Feuerriegel_2023} is closely related to the present study which uses deep neural networks to estimate the whole propensity scores and outcome regressions simultaneously. Furthermore, it has an additional layer for targeting implementing the one-dimensional submodel proposed by van der Laan \cite{van2018targeted}. Our method differs from theirs in the following three aspects. First, DeepACE incorporates the targeting mechanism within their loss function, which requires an additional hyperparameter. However, there is a lack of justification for the chosen value of this hyperparameter and guidance on its tuning in practical applications. Our approach, in contrast, separates the targeting step, aligning more closely with the TMLE literature. Second, DeepACE does not address survival outcomes, specifically failing to consider the process degeneracy following a patient's event occurrence. Third, while DeepACE utilizes the long short-term memory (LSTM) architecture, our method employs transformers. Transformers are superior in capturing long-term dependencies and offer greater computational efficiency during training than LSTM. Moreover, DeepACE does not provide uncertainty measures, such as confidence intervals, limiting its utility for statistical inference.

Our problem of estimating mean of counterfactual outcomes from longitudinal observational data under dynamic interventions has been extensively investigated as an off-policy evaluation problem in the bandit algorithm and reinforcement learning literature \cite{levineORL}. Methods of bias correction after plugging in the initial estimate with influence function were also introduced in this context \cite{jiangDROPE,farajtabarMRDROPE,naritaDOPE}. However, they did not provide tools for inference. Double reinforcement learning \cite{ueharaDRL} utilized the efficient influence functions in the spirit of double machine learning \cite{chernozhukovDML}, which is a closed form of a more general debiased estimating equation framework \cite{cherLocally}, to correct plug-in bias and proved efficiency. TMLE deform the distribution itself to correct bias before plugged-in to the the target functional, thereby the values are contained the domain of the functional. Since the estimator based on estimating equation was reported to have inferrior performance to LTMLE even in short time-horizons \cite{tran_double_2019}, we focus on plug-in estimators in the present study.

\section{Problem Formulation}
In this section, follwing the roadmap of causal inference \cite{petersenCausalModelsLearning2014,van2018targeted,dangCausalRoadmapGenerating2023}, we first describe the causal structure that generated the observed data and the statistical model that contains the data-generating distribution. Next, we define our causal target parameter. Then, we discuss assumptions needed to identify our target parameter from the observed data. Finally, we describe the idea of statistical method for constructing estimator and correcting bias.

\subsection{Data}
We consider the general longitudinal setting involving repeated measurements of a set of variables for a group of $n$ patients over a period of time. In particular, our observed data contains $n$ independent and identically distributed copies of random vector 
\begin{align}
    O=(W_0 = W, L_1, A_1, Y_1, \ldots, L_T, A_T, Y_T = Y)
\end{align}
with baseline covariates $W$, time-dependent covariates $L_t$, treatments $A_t$, and outcome $Y_t$. We use $P_0$ to denote the true probability distribution of $O$ that generated the data, and $P_0$ is in some statistical model $\mathcal{M}$. Stopping time $T$ is a random variable (e.g. time of death in the case of survival analysis) and we use $\tau$ to denote the maximum time. We make the remark that in real-world data, patients are often subject to censoring. For a formulation of the data structure involving censoring nodes, see Appendix \ref{sec:extended-model}.

\subsection{Target Parameter}

To define the target parameter, we introduce a structural causal model (SCM). In brief, SCM assumes each observed random variable $X$ is generated from the parent nodes $pa(X)$ and the external noise $U_X$ by a production function $f_X$ as $X = f_X(pa(X), U_X)$. By abusing notation, we also denote the induced probability measure of $X$ by the same symbol $f_X$. See Appendix \ref{sec:scm} for details.

Our target parameter is the counterfactual mean of the final outcome $Y$ under a user-specified dynamic treatment policy $g=[g_t]_{t=1}^\tau$ where $g_t$ is a probability measure on the treatment space conditioned on the whole history, $pa(A_t) = (L_{1:t},A_{1:t-1},Y_{1:t-1})$ up until $A_t$ (not including $A_t$). Specifically, our target parameter is given by
\begin{equation}
    \psi(P) = \EE Y^g,
\end{equation}
which is the mean of the counterfactual outcome produced by replacing $\pi$, defined as the behavior treatment policy observed on the data, with $g$ in the structural causal model. 

\paragraph{Identification} Under the positivity assumption: 
\begin{equation}
    g \ll \pi,    
\end{equation}
and the sequential randomization assumption:
\begin{equation}
    Y^g \perp A_t \mid pa(A_t) \text{ for } t=1,\ldots,\tau,    
\end{equation}
we can identify our target causal parameter through g-formula as the mean of $Y$ under the counterfactual distribution which is given by replacing distributions $\pi$ with $g$ \cite{robins1986new}:
\begin{equation}
    \EE Y^g = \EE_gY.
\end{equation}
Note that the consistency assumption $Y=Y^{\pi}$, usually stated in causal inference literature, is a consequence of the definition of counterfactual outcome in our SCM. Now the problem is reduced to the estimation of the statistical parameter:
\begin{equation}
    \label{eq:g-functional}
    \psi(P) = \EE_gY.
\end{equation}

\subsection{Targeted Minimum Loss-based Estimation}

\label{sec:plug-in}

Given we have an estimator $\hat{P}_n$ of the data-generating distribution $P_0$, a natural estimator of the target functional is the plug-in estimator $\hat{\psi}_n=\psi(\hat{P}_n)$. Under a regularity condition, $\psi$ admits the following first-order expansion
\begin{equation}
    \psi(\hat{P}_n)-\psi(P_0)=-\int_{\mathcal{O}}D^\star(\hat{P}_n)dP_0 + R_2(\hat{P}_n,P_0),
\end{equation}
where $D^\star$ is the efficient influence function of $\psi$, and $R_2(\hat{P}_n,P_0)$ is the exact remainder. Influence functions quantifies the amount of changes of an estimator under small perturbations of the sample. The efficient influence function is the influence function with minimal variance. The idea of TMLE is to eliminate the empirical analogue of the first term of the right hand side by fluctuating $\hat{P}_n$ to find a distribution $\hat{P}^\star_n$ with $P_nD^\star(\hat{P}^\star_n)=0$, where $Pf=\int fdP$ is a shorthand for the expectation of a measurable function $f$ with respect to a probability measure $P$. Our problem is to obtain an initial estimate $\hat{P}_n$ with a potentially large scale and high dimensional longitudinal data, and correct bias of the plug-in estimator $\psi(\hat{P}_n)$ by fluctuating $\hat{P}_n$. 

\begin{algorithm*}[tb]
   \caption{Temporal Difference Learning of Conditional Counterfactual Mean Outcomes}
   \label{alg:td-learning}
\begin{algorithmic}[1]
   \FOR{$b=1$ {\bfseries to} $B$}
       \STATE $\hat{Q}_T(pa(Y_T)) \leftarrow \hat{Q}_T(pa(Y_T)) - \alpha \cdot \partial_{\, \hat{Q}_T(pa(Y_T))}{\mathcal{L}(\hat{Q}_T(pa(Y_T)), \; Y_T )}$
       \STATE $\hat{\pi}_T\leftarrow \hat{\pi}_T- \alpha \cdot \partial_{\, \hat{\pi}_T(pa(A_T))}{\mathcal{L}(\hat{\pi}_T(pa(A_T)), A_T)}$
       
       \FOR{$t=T-1$ {\bfseries to} $1$}
       \STATE $\hat{V}_{t+1}(pa(A_{t+1})) \leftarrow \mathbb E_{A_{t+1} \sim g_{t+1}(pa(A_{t+1}))} \big[\hat{Q}_{t+1}(pa(A_{t+1}), A_{t+1})\big]$
       \STATE $\hat{Q}_{t}(pa(Y_t)) \leftarrow \hat{Q}_{t}(pa(Y_t)) - \alpha \cdot \partial_{\hat{Q}_{t}(pa(Y_t))} \mathcal{L}(\hat{Q}_{t}(pa(Y_t)), \; \hat{V}_{t+1}(pa(A_{t+1})) )$
       \STATE  $\hat{\pi}_t\leftarrow \hat{\pi}_t- \alpha \cdot \partial_{\, \hat{\pi}_t(pa(A_t))}{\mathcal{L}(\hat{\pi}_t(pa(A_t)), A_{t})}$
       \ENDFOR
   \ENDFOR
   \STATE Output $(\hat{Q}_t, \hat{V}_t, \hat{\pi}_t)_{t=1}^T$
\end{algorithmic}
\end{algorithm*}

\section{Proposed Method}

\label{sec:Deep LTMLE}

In this section, we describe our proposed method, Deep Longitudinal Targeted Minimum Loss-based Estimation (Deep LTMLE). Let 
\begin{align}
    Q_t(pa(Y_t)) = \mathbb E_g[Y_\tau \mid L_{1:t}, A_{1:t},Y_{1:t-1}]
\end{align}
be the mean outcome at the end of the observation period $\tau$ given the history before node $Y_t$ for $t=1,\ldots,\tau$, where future treatments $A_{t+1}, \ldots, A_\tau$ follow treatment assignment policy $g$. Similarly, 
\begin{align}
    V_t(pa(A_t)) = \mathbb E_g[Y_\tau \mid L_{1:t},A_{1:t-1},Y_{1:t-1}]
\end{align}
is the mean outcome at the time $\tau$ given the history before node $A_t$, for $t=1,\ldots,\tau$. We abbreviate $Q_t$ for $Q_t(pa(Y_t))$ if it is clear from the context, similarly for $V_t$. Our goal is to estimate $\psi(P)$ by
\begin{align}
    \label{final-step-plug-in}
    \hat{\psi}^\star_{n}&=P_n\hat{V}_{1,\varepsilon^\star}(pa(A_1)) \nonumber\\
    &= P_n \mathbb E_{A_1 \sim g_1(pa(A_1))} [\hat{Q}_{1,\varepsilon^\star}(pa(A_1), A_1)]
\end{align}
where $\hat{Q}_{1,\varepsilon^\star}$ is an estimation of $Q_1$ such that $\hat{\psi}^\star_{n}$ is asymptotically efficient. We achieve this by proposing a temporal-difference heterogeneous transformer to yield an initial estimation $\hat{Q}_1$, then update this estimation to get $\hat{Q}_{1,\varepsilon^\star}$ via Targeted Minimum Loss-based Estimation (TMLE). We denote the stopping time by $T$ when we observe failure or censoring, or the trajectory reached the final time $\tau$.

\begin{figure*}[ht]
    \centering
    \includegraphics[width=\textwidth]{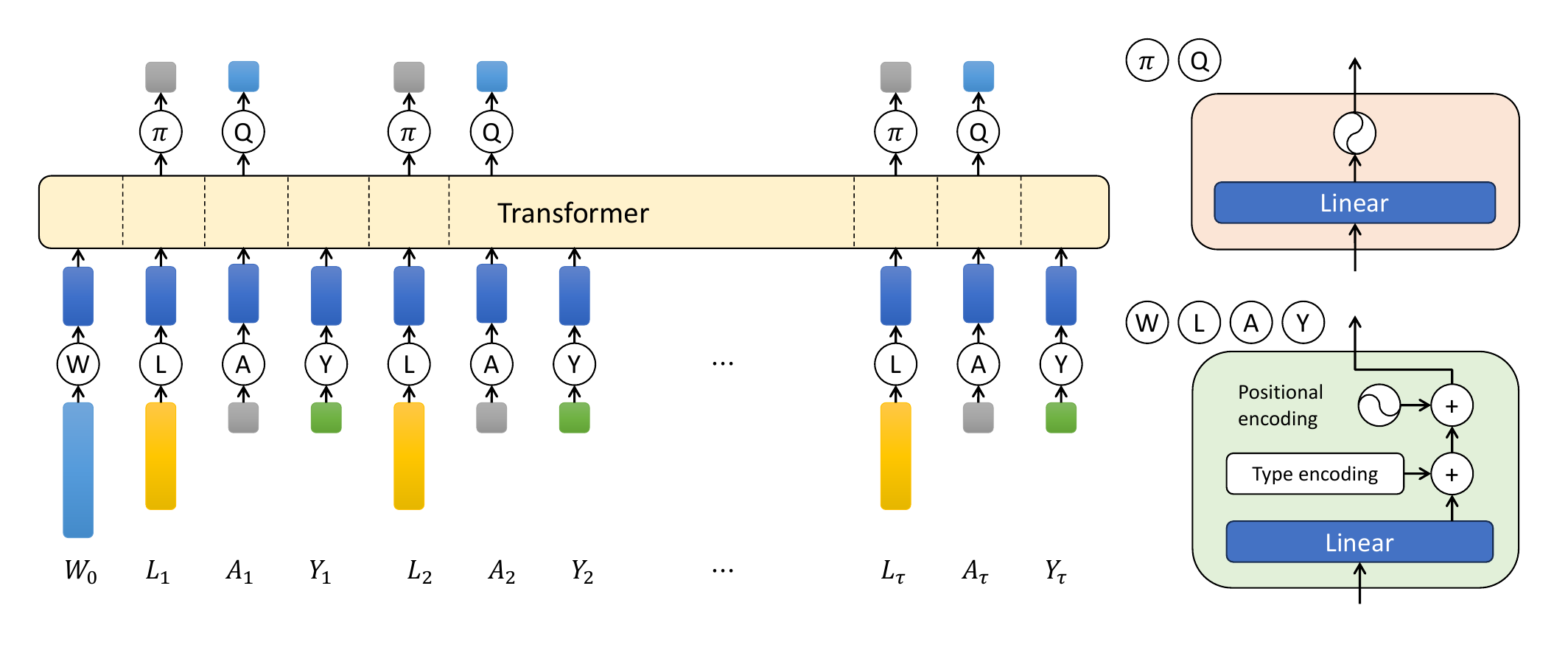}
    \caption{{\bf Architecture of temporal-difference heterogeneous-token transformer (TDHT)}. Observed variables are fed into transformer after embedding layers depending on the variable types. Embedding layers aggregate linear transform with learnable type encoding and learnable positional encoding. Outputs of the transformer are $\pi$ after $L$ and $Q$ after $A$. Each output head consists of a linear layer and the final activation function depending on variable distribution (sigmoid for binary, softmax for categorical and none for continuous). The outputs of $\pi$ heads are used to learn propensity scores and the outputs of $Q$ heads are used for temporal-difference learning after integration with respect to the counterfactual treatment policy.}
    \label{fig:architecture}
\end{figure*}

\subsection{Temporal-Difference Heterogeneous Transformer}

To learn the initial model $\hat{V}_1$, we use temporal-difference loss as the objective to learn underlying models $\hat{Q}_t$ for $t=1, \ldots, \tau$ via stochastic gradient descent (SGD). The principle of temporal difference learning \citep{sutton1988learning, mnih2013playing} is to supervise $\hat{Q}_t$ to obey the temporal equality of $Q_t$:
\begin{align}
    Q_t &= \mathbb E_{p(Y_t, L_{t+1} \mid pa(Y_t))} V_{t+1} \nonumber\\
    &= \mathbb E_{p(Y_t, L_{t+1} \mid pa(Y_t)), \, g_{t+1}(pa(A_{t+1}))} Q_{t+1} 
\end{align}

for $t=1, \ldots, \tau-1$ and $Q_\tau = \mathbb E_{p(Y_\tau \mid pa(Y_\tau))} Y_\tau$. The temporal difference loss on a sample trajectory is thus given by 
\begin{align}
    \mathcal{L}^Q_t = \mathcal{L}(\hat{Q}_t, \hat{V}_{t+1}) = \mathcal{L}(\hat{Q}_t, \mathbb E_{g_{t+1}(pa(A_{t+1}))} \big[\hat{Q}_{t+1}\big])
\end{align}
for $t=1, \ldots, T-1$, where $\mathbb E_{g_{t+1}(pa(A_{t+1}))} \big[\hat{Q}_{t+1}\big]$ can be computed by Monte-Carlo estimation if $g$ is a density function, and $\mathcal{L}^Q_T = \mathcal{L}(\hat{Q}_T, Y_T)$. In the case of survival analysis, the components for $t=1,\ldots,\tau-1$ are defined as
$\mathcal{L}^Q_t = \mathcal{L}_{\mathrm{bce}}(\hat{Q}_t, \hat{V}_{t+1})$, where $\mathcal{L}_{\mathrm{bce}}(\hat{y},y)=-\big[y\log \hat{y}+(1-y)\log(1-\hat{y})\big]$  is the binary cross entropy loss. To yield the updated model $\hat{Q}_{t,\varepsilon}$, we need to adjust $\hat{Q}_t$ post training to factor in estimations of propensity scores $e_t(a) = \pi_t(a | pa(A_t))$, which we will describe in detail in the next section. Hence, the loss function also includes 
\begin{align}
    \mathcal{L}^e_t = \mathcal{L}_{\mathrm{bce}}(\hat{\pi}_t, A_t)
\end{align}
for estimating model $\hat{\pi}_t$ and is thus given by
\begin{align}
    \mathcal{L} 
    = \sum_{t=1}^{\tau} \mathbbm{1}\{t \leq T\} \mathcal{L}_t 
    = \sum_{t=1}^{T} \mathcal{L}^{Q}_t + \alpha \mathcal{L}^e_t
\end{align} 
where $\alpha$ is a scaling hyperparameter. See Algorithm \ref{alg:td-learning} for the optimization workflow. Convergence of the algorithm can be found in Appendix \ref{sec:convergence of td}.

For the estimation of $Q_t$ and $\pi_t$, we propose a unified model architecture to simultaneously optimize $\hat{Q}_t$ and $\hat{\pi}_t$ in an efficient, non-sequential manner by adapting a decoder-only Transformer \cite{vaswani2017a, brown2020language} to longitudinal data with heterogeneous tokens. An overview of the model architecture is given in Figure \ref{fig:architecture}. For each sampled sequence in the training set, we feed each token in the sequence to an embedding layer $h_\bullet$, which can be type-specific linear layers as in $h_W$, $h_L$, $h_A$, $h_Y$, or simply a shared layer $h_\bullet = h$ padding each token to the same length for a more efficient and parallelizable embedding operation. Then, each embedding is aggregated with its positional encoding $H_0, \ldots, H_\tau$ and type encoding $H_W$, $H_L$, $H_A$, and $H_Y$ that represent its timestamp and variable type information:
\begin{align}
    H(\bullet_t) = \texttt{concat}\left\{ h_\bullet(\bullet_t), \, H_\bullet, \, H_t \right\}
\end{align}
for $\bullet_t \in (W_0, L_1, A_1, Y_1, \ldots, L_\tau, A_\tau, Y_\tau)$. Note that we include type embedding $H_\bullet$ because $h_\bullet$ need not necessarily be type-specific. Then, the embedded sequence is fed into the transformer and produce $\hat{\pi}$ and $\hat{Q}$ through output layers $f_e$ and $f_Q$ at each position that corresponds to token type $L$ and $A$ respectively:
\begin{align}
    \hat{\pi}_t &= f_e(\mathrm{transformer}\left\{H(W_0), \, \ldots, \, H(L_t)\right\}) \\
    \hat{Q}_t &= f_Q(\mathrm{transformer}\left\{H(W_0), \, \ldots, \, H(A_t)\right\}).
\end{align}

In practice, we can use a joint output layer $f$ for $f_e$ and $f_Q$ for more efficient and parallelizable output generation, where the output number of dimensions is the sum of the number of dimensions $\mathrm{dim}(A)$ for treatment $A$ and $\mathrm{dim}(Y)$ for outcome $Y$. Then, we compute softmax probabilities masking out the last $\mathrm{dim}(Y)$ dimensions for $\hat{\pi}_t$ and first $\mathrm{dim}(A)$ dimensions for $\hat{Q}_t$.

Our proposed architecture does not entail concatenation of variables at the same timestamp or sequential decoding of outputs following the transformer blocks like prior work \citet{pmlr-v162-melnychuk22a}, which 1) allows us to handle different types of and different number of variables at different timestamps (e.g. starting from $W_0$, ending at $L_8$, while $A_3$ and $Y_3$ are missing), and 2) is fully parallelizable when we use padding instead of learnable linear mapping for the embedding layer $h_\bullet$ and use the joint output layer $f$.

\subsection{Targeted Minimum Loss-based Estimation}



\paragraph{Efficient Influence Function} Since our target parameter is the counterfactual mean outcome at the final $\tau$, the relevant part of $P_0$ of interest are $Q_{t,0}$ for $t=1,2,...,\tau$. 

\begin{theorem}
In our counterfactual mean case, the efficient influence function $D^\star(P)(O)=D^\star(\{Q_{t},\pi_t\}_{t=1}^\tau)(O)$ is given by
\begin{align}
      D^\star(P)(O) = (V_1 - \psi_0)+\sum_{t=1}^T   I_{t}(V_{t+1} - Q_t)
\end{align}
where $I_t=\prod_{s=1}^{t}dg_s/d\pi_s$ and $V_{T+1} = Y_T$.    
\end{theorem}

This is given in \cite{laanTargetedMinimumLoss2012}.

\subsubsection{Temporal Difference Targeting}
\paragraph{Submodel}
We update the initial estimate $\hat{Q}_t$ for $Q_{t,0}$ to $\hat{Q}^{g\star}_t$ such that $P_n D^\star(\{\hat{Q}^{g\star}_{t},\hat{\pi}_t\}_{t=1}^\tau)=0$. We realize this by fluctuating $\hat{Q}_t$ along a one-dimensional submodel through the initial fit $\hat{Q}=[\hat{Q}_t]_{t=1}^\tau$ given by, $\hat{Q}_\varepsilon=[\hat{Q}_{t,\varepsilon}]_{t=1}^\tau$, where
\begin{align}
    \logit \hat{Q}_{t,\varepsilon} = \logit \hat{Q}_t + \varepsilon
\end{align}
with a common fluctuation parameter $\varepsilon$ across $t$. If the outcome is survival, then we automatically have $Y_t\in[0,1]$. In a general longitudinal setting for bounded $Y_t$'s, we can re-scale both $Y_t$ and $\hat{Q}_t$ to $[0,1]$ and use the same one-dimensional submodel.

\paragraph{Partial Loss function}
We search for the optimal fluctuation $\varepsilon^\star$ with respect to the partial loss function 
\begin{align}
    \mathcal{L}^\star(\hat{Q}_\varepsilon, \hat{V}_{\varepsilon'};\hat{\pi})=\sum_{t=1}^T I_{t}(\hat{\pi})\mathcal{L}_{\mathrm{bce}}(\hat{Q}_{t,\varepsilon}, \hat{V}_{t+1,\varepsilon'}),
\end{align}
where $\hat{V}_{T+1,\varepsilon}=Y_T$ and $\hat{V}_{t,\varepsilon}= \mathbb E_{A_{t} \sim g_{t}} \big[\hat{Q}_{t,\varepsilon}\big]$, such that $\mathcal{L}^\star(\hat{Q}_\varepsilon, \hat{V}_{\varepsilon'})$ satisfies the following theorem:

\begin{theorem}
    \label{thm:plf}
    For any $\varepsilon^\star$, we have
    \begin{align}
        \partial_{\varepsilon}|_{\varepsilon^\star}\mathcal{L}^\star(\hat{Q}_\varepsilon, \hat{V}_{\varepsilon^\star}) = D^\star(\hat{Q}_{\varepsilon^\star}, \hat{\pi}).
    \end{align}
\end{theorem}

See Section \ref{sec:proof} for the proof.

\begin{corollary} 
Suppose that we found an $\varepsilon^\star$ satisfying 
\begin{equation}\label{criteria for epsilon star}
    \partial_{\varepsilon}|_{\varepsilon^\star}P_n\mathcal{L}^\star(\hat{Q}_\varepsilon, \hat{V}_{\varepsilon^\star})=0,
\end{equation} then $\hat{Q}_{\varepsilon^\star}$ solves the efficient influence function.
\end{corollary}

In practice, for the finite sample performance, we only need to solve it to the order of standard error with a $\sigma_n/\log n$ factor \cite{laanTLBook2011} (Algorithm \ref{alg:targeting}).

\begin{algorithm}
    \caption{Temporal-Difference Targeting}
    \label{alg:targeting}
\begin{algorithmic}[1]
   
   \STATE $\varepsilon\leftarrow 0$
   \REPEAT
   \STATE$\varepsilon\leftarrow\mathrm{argmin}_{\varepsilon'} P_n\mathcal{L}^\star(\hat{Q}_{\varepsilon'}, \hat{V}_{\varepsilon}).$
   \STATE $\hat{\sigma}_n\leftarrow \sqrt{n^{-1}P_nD^{\star2}(\hat{Q}_{\varepsilon}, \hat{\pi})}$
   \UNTIL{$P_nD^\star(\hat{Q}^{g}_{\varepsilon},\hat{\pi})
   <\hat{\sigma}_n/\log n$}
   \STATE $\hat{\psi}^\star_{n}\leftarrow P_n\hat{V}_{1,\varepsilon}(pa(A_1))$
   \STATE 95\% CI as $\hat{\psi}^\star_{n}\pm 1.96 \cdot \hat{\sigma}_n$
\end{algorithmic}
\end{algorithm}

\paragraph{Convergence of Algorithm \ref{alg:targeting}}
The investigation of $\mathcal{L}_{\mathrm{bce}}(\hat{Q}_{t,\varepsilon}, \hat{V}_{t,\varepsilon})(O)$ for different $t$ and $O$'s as a function of $\varepsilon$ suggests that they admit different bell curve shapes concentrating at different $\varepsilon$'s and have different spread out levels. Thus, by summing up $\mathcal{L}_{\mathrm{bce}}(\hat{Q}_{t,\varepsilon}, \hat{V}_{t,\varepsilon})(O)$ across $t$ and across $O$'s as $P_n\mathcal{L}^\star(\hat{Q}_{\varepsilon}, \hat{V}_{\varepsilon})$ as a function of $\varepsilon$ will fluctuate a lot and we expect a local minima and local maxima around the neighborhood of $\varepsilon=0$. And thus the convergence of the algorithm is highly probable and we don't discover any issue in our simulations.

\paragraph{Comparison to LTMLE}
In the LTMLE, we only need a good estimate of ${Q}_\tau$ and then do backward sequential regression and targeting as mentioned in \cite{laanTargetedMinimumLoss2012}. However, the problem is the error in the estimation of ${Q}_\tau$ can propagate as we progress back to get ${Q}^{*}_{\tau-1},...,{Q}^{*}_{1}$. Nonetheless, after our initial transformer step, we have good initial estimates for all ${Q}_{1},...,{Q}_{\tau}$. So, instead of only relying on a good estimate of ${Q}_\tau$, our algorithm makes uses of all of them. and doing targeting across $t$ with $o(\varepsilon)$ fluctuation at each $t$ level. Thus, we are able to pool information across time when doing the targeting step.

\subsubsection{Sequential Targeting}
Alternatively, one could apply a sequential targeting procedure that is very similar to LTMLE but with given initials generated from the transformer step.

\paragraph{Submodel} We fluctuate each component of the initial fit $\hat{Q}=[\hat{Q}_{t}]_{t=1}^\tau$ along a model as 
\begin{align}\label{submodel_t}
    \logit \hat{Q}_{t,\varepsilon_t} = \logit \hat{Q}_t + \varepsilon_t.
\end{align}

\paragraph{Loss Function}
Starting from $t=\tau$, given we have found $\varepsilon_{t+1}^\star$, among individuals whose $T>t-1$, we search for empirical loss minimizer $\varepsilon_t^\star$ with respect to the loss function $\mathcal{L}^\star_t$ as,
\begin{align}
    \mathcal{L}^\star_t(\hat{Q}_{t,\varepsilon_t}, \hat{V}_{t+1,\varepsilon_{t+1}^\star})=I_{t}(\hat{\pi})\mathcal{L}_{\mathrm{bce}}(\hat{Q}_{t,\varepsilon_t}, \hat{V}_{t+1,\varepsilon_{t+1}^\star}),
\end{align}
where $\hat{V}_{t+1,\varepsilon_{t+1}^\star}= \mathbb E_{A_{t+1} \sim g_{t+1}} \big[\hat{Q}_{t+1,\varepsilon_{t+1}^\star}\big]$ when $T>t$ and $\hat{V}_{t+1,\varepsilon_{t+1}^\star}=Y_T$ when $T=t$. To initialize, we set $\varepsilon_{\tau+1}^\star=0$.

\begin{lemma} 
Suppose that we found $\varepsilon^\star_\tau,...\varepsilon^\star_1$ sequentially as mentioned above, then $\{\hat{Q}_{\varepsilon_t^\star}\}_{t=1}^\tau$ solves the efficient influence function. 
\end{lemma} 

\begin{algorithm}
    \caption{Sequential Targeting}
    \label{alg:seq-targeting}
\begin{algorithmic}[1]
   
   \STATE $\varepsilon_{\tau+1}^\star\leftarrow 0$
   \FOR{$t=\tau$ {\bfseries to} $1$}
   \STATE$\varepsilon_t^\star\leftarrow\mathrm{argmin}_{\varepsilon_t} P_n \mathbbm{1}\{T\geq t\}\mathcal{L}^\star(\hat{Q}_{t,\varepsilon_t}, \hat{V}_{t+1,\varepsilon_{t+1}^\star}).$
   \ENDFOR
   \STATE $\hat{\psi}^\dagger_{n}\leftarrow P_n\hat{V}_{1,\varepsilon_1^\star}(pa(A_1))$
   \STATE $\hat{\sigma}_n\leftarrow \sqrt{n^{-1}P_nD^{\star2}(\{\hat{Q}_{\varepsilon_t^\star}\}_{t=1}^\tau, \hat{\pi})}$
   \STATE 95\% CI as $\hat{\psi}^\dagger_{n}\pm 1.96 \cdot \hat{\sigma}_n$
\end{algorithmic}
\end{algorithm}

\paragraph{Comparaison to LTMLE}
While the error can still propagate as we move back in time,  the error propagates only through the targeting steps whereas in LTMLE the error can also propagate through regressions. At each time step $t$, LTMLE needs to first regress $\hat{V}_{t+1,\varepsilon^\star_{t+1}}$ on $pa(Y_t)$ to get an estimate $\hat{Q}_t$ and then perform the targeting through the submodel in(\ref{submodel_t}). However, we only use initial estimate $\hat{Q}_t$ from our transformer fit and it does not depend on $\varepsilon^\star_{t+1},\ldots,\varepsilon^\star_{\tau}$.

\paragraph{Why not targeting through additional loss function}
As in DeepACE, the targeting can be performed through introducing additional loss components to further train the transformer we have build in the first step. This additional loss function will have its derivative equal to the efficient influence function. However, we find that the penalty factor before this loss function is hard to tune and in near all cases, it is hard to guarantee the EIF is solved and most of the time we will hurt our initial fits as shown in Appendix \ref{sec:continuous outcome experiment}.

\begin{figure*}[ht]
    \centering
    \includegraphics[width=\textwidth]{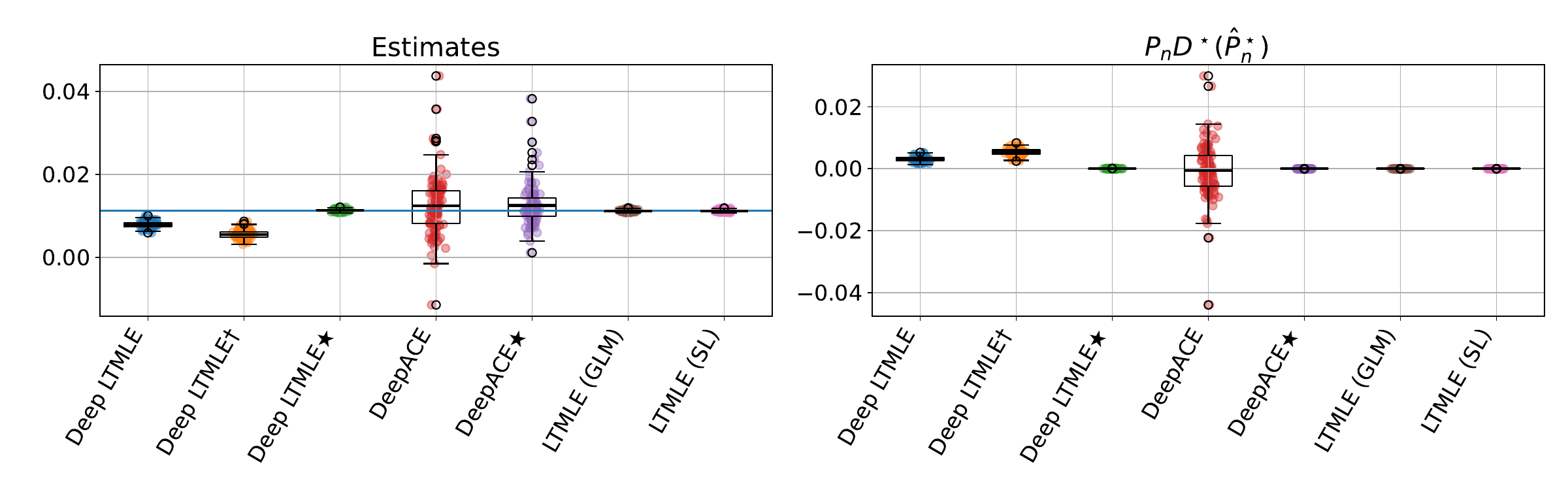}
    \caption{{\bf Results from simple synthetic data with continuous outcome.} Left: Sampling distributions of estimates. Right: Sampling distributions of empirical means of estimated efficient influence functions. Deep LTMLE: initial estimate with TDHT; Deep LTMLE$\dagger$: TDHT with sequential targeting; Deep LTMLE$\star$: TDHT with temporal-difference targeting; DeepACE$\star$: DeepACE with temporal-difference targeting; LTMLE (GLM): LTMLE with GLM; LTMLE (SL): LTMLE with super learner of GLM, MARS, and XGBoost.}
    \label{fig:continuous}
\end{figure*}

\begin{table*}[ht]
\centering
\caption{Results from complex synthetic data.}
\label{table2}
\resizebox{\textwidth}{!}{
    \begin{tabular}{lrrrrrrcccrrr}
    \toprule
     & \multicolumn{3}{c}{Bias} & \multicolumn{3}{c}{RMSE} & \multicolumn{3}{c}{Coverage} & \multicolumn{3}{c}{Mean $\hat{\sigma}_n$} \\
    Model & $\tau = 10$ & $\tau = 20$ & $\tau = 30$ & $\tau = 10$ & $\tau = 20$ & $\tau = 30$ & $\tau = 10$ & $\tau = 20$ & $\tau = 30$ & $\tau = 10$ & $\tau = 20$ & $\tau = 30$ \\
    \midrule
    LTMLE (GLM) & 0.0230 & 0.0766 & 0.1344 & 0.0265 & 0.0796 & 0.1381 & 1.00 & 1.00 & 1.00 & 0.43 & 0.69 & 0.76 \\
    LTMLE (SL) & 0.0144 & 0.0297 & 0.0477 & 0.0185 & 0.0344 & 0.0545 & 1.00 & 1.00 & 1.00 & 0.31 & 0.40 & 0.45 \\
    DeepACE & 0.0055 & -0.0211 & -0.0714 & 0.0397 & 0.0672 & 0.0962 & 1.00 & 1.00 & 1.00 & 0.69 & 0.76 & 0.73 \\
    \midrule
    Deep LTMLE & 0.0181 & 0.0292 & 0.0503 & 0.0263 & 0.0332 & 0.0533 & 0.99 & 0.96 & 0.75 & 0.17 & 0.10 & 0.06 \\
    Deep LTMLE$\dagger$ & 0.0319 & 0.0327 & 0.0535 & 0.0571 & 0.0379 & 0.0571 & 0.99 & 0.97 & 0.78 & 0.19 & 0.10 & 0.07 \\
    Deep LTMLE$\star$ & 0.0156 & 0.0307 & 0.0496 & 0.0218 & 0.0344 & 0.0541 & 0.99 & 0.95 & 0.78 & 0.17 & 0.09 & 0.07 \\
    \bottomrule  
    \end{tabular}
}
\end{table*}

\section{Experiments}
We conducted two experiments. In the first experiment, we compare the bias, root-mean-squared-error (RMSE), and coverage probability, of our estimator with existing estimators based on 100 times of estimations for both continuous and survival outcomes. The second experiment is an application of our proposed method to a real-world data. 

\subsection{Synthetic Data with Continuous Outcome}

\label{sec:continuous outcome experiment}

First, we start our experiment with a very simple data generating process with continuous outcome, $n=500$, and $\tau=10$. The data generating proccess is described in the Section \ref{sec:continuous synthetic data}. After fitting DeepACE, we additionally performed our targeting precedures on the fit. 

The results were shown in Figure \ref{fig:continuous}. Initial fits of Deep LTMLE and DeepACE had comparable bias. Even with the targeting loss, DeepACE failed to solve the efficient influence function. On the other hand, due to the separation of the targeting step in our method, we managed to solve it completely and succeeded in correcting bias.

\subsection{Synthetic Data with Survival Outcome}

\begin{table*}[ht]
\centering
\caption{Results from semi-synthetic data with unmeasured confounding}
\label{table-r2}
\resizebox{\textwidth}{!}{
    \begin{tabular}{lrrrrrrcccccc}
    \toprule
     & \multicolumn{3}{c}{Bias} & \multicolumn{3}{c}{RMSE} & \multicolumn{3}{c}{Coverage} & \multicolumn{3}{c}{Mean $\hat{\sigma}_n$} \\
    Model & $\tau = 10$ & $\tau = 20$ & $\tau = 30$ & $\tau = 10$ & $\tau = 20$ & $\tau = 30$ & $\tau = 10$ & $\tau = 20$ & $\tau = 30$ & $\tau = 10$ & $\tau = 20$ & $\tau = 30$ \\
    \midrule
    LTMLE (SL) & 0.0075 & 0.0341 & 0.0574 & 0.0138 & 0.0491 & 0.0786 & 0.70 & 0.45 & 0.25 & 0.09 & 0.12 & 0.14 \\
    DeepACE & -0.0174 & -0.0434 & -0.0770 & 0.0788 & 0.1154 & 0.1341 & 1.00 & 1.00 & 1.00 & 0.67 & 0.78 & 0.86 \\
    \midrule
    Deep LTMLE & -0.0002 & 0.0108 & 0.0041 & 0.0162 & 0.0720 & 0.0772 & 1.00 & 0.95 & 1.00 & 0.18 & 0.27 & 0.32 \\
    \bottomrule
    \end{tabular}
}
\end{table*}

Next, we evaluated Deep LTMLE under a highly complex data-generating process with survival outcomes, five-dimensional time-dependent covariates, non-Markovian dependencies, $n=1000$, and $\tau=10,20,30$, imitating the setups from previous studies \cite{Bica2020Estimating, Frauen_Hatt_Melnychuk_Feuerriegel_2023}. See Section \ref{sec:complex synthetic data} for details.

Results are presented in Table \ref{table2}. We observe that Deep LTMLE and LTMLE with a super learner on average achieves a lower RMSE compared to other methods, particularly in scenarios with larger $\tau$, indicating its robustness in complex and realistic scenarios without Markovian dependencies. Benefits by our temporal difference targeting procedures are obvious for $\tau=10$. While Deep LTMLE's coverage probability diminished at $\tau=30$, the confidence intervals generated by LTMLE and DeepACE were notably over-conservative with large estimated standard errors. 

The pronounced bias of DeepACE can likely be attributed to three factors. First, DeepACE's use of the squared-error-loss for the outcome is known to induce greater bias in sparse outcomes, a common scenario in survival analysis, as opposed to the logistic loss used in our approach \cite{gruber_targeted_2010}. Second, DeepACE failed to solve the efficient influence function. Third, DeepACE does not account for the degeneration of the survival outcome. 

\paragraph{Simple Synthetic Data with Survival Outcome}
We also conducted an eperiment with a very simple survival synthetic data with one-dimensional time-dependent covariates, $n=1000$, and $\tau=10,20,30$. Although LTMLE with GLM is expected to have strong performance in this experiment, Deep LTMLE remains highly competitive in this context, equalling LTMLE’s performance (Section \ref{sec: results with simple synthetic survival data}).

\subsection{Semi-Synthetic Data}

To evaluate the performance of the proposed methods, we generated realistic data from Circulatory Risk in Communities Study (CIRCS) \cite{yamagishi2019}, a long-term on-going cardiovascular epidemiological cohort study, lasting over a half century. See Section \ref{sec:semi-synth} for the detail.

Table \ref{table-r2} shows the results with semi-synthetic data with unmeasured confounding, which reflects a real world setting. Deep LTMLE performed best in terms of bias for all time horizons. Furthermore, as the time horizon increases from 10 to 30, LTMLE’s coverage probability drops as low as 0.25. On the other hand, Deep LTMLE has nominal coverage even in the longest time-horizon setting. 

\subsection{Real World Data}

\label{sec:rwd}

We applied Deep LTMLE to real world data from CIRCS. We estimated the counterfactual mean outcomes under the standard blood pressure (SBP) management strategy that controls SBP less than 140 mmHg and the intensive blood pressure management strategy with SBP less than 120 mmHg after the 30 years of sustained management. 

In real world applications, we often encounter with practical problems of censoring, that is loss of follow-up for some reasons. Our model can be easily generalized to cover this setting with a slight modification by adding censoring nodes. Details are described in Section \ref{sec:extended-model} of Appendix.

The results were shown in Figure \ref{fig:circs}. The average treatment effect (ATE) of the intensive management strategy over the standard management strategy first increased with a peak at 20 years after baseline and then decreased with a fluctuation. The direction and trend of ATE is consistent with the difference of empirical means of cumulative outcomes between two groups followed the two strategies. 

\begin{figure}[ht]
    \includegraphics[width=0.9\columnwidth]{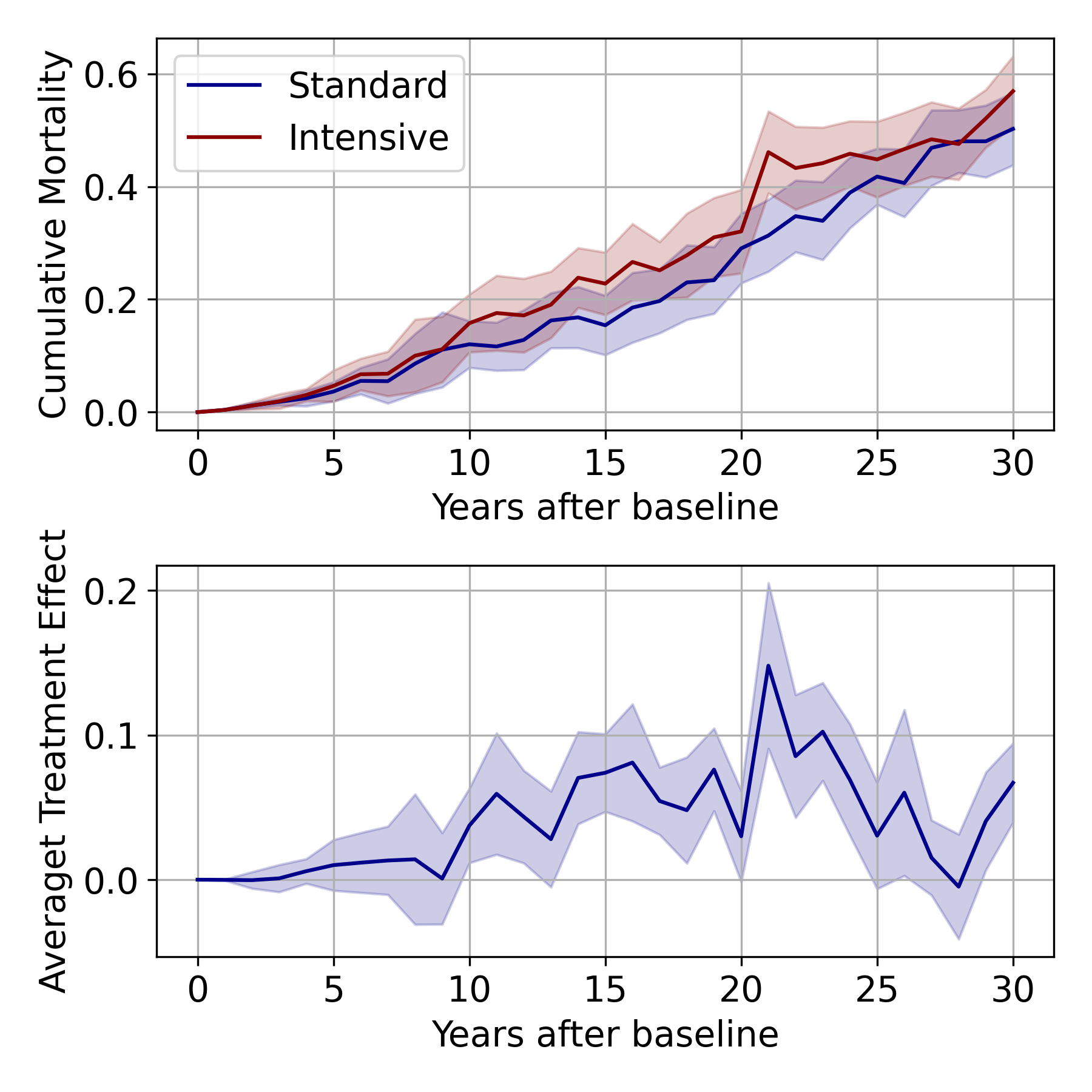}
    \caption{\small {\bf Counterfactual mean outcomes and 95\% simultaneous confidence intervals according to standard and intensive treatment policies among 13,485 participants of the CIRCS.} }
    \label{fig:circs}
\end{figure}

\subsection{Computation Details}
DeepACE and Deep LTMLE were run on a GPU (Tesla T4) with 16 GB memory and LTMLE on CPU (Intel Xeon Skylake 6230 @ 2.1 GHz) with 40 cores and 96 GB memory. We used the R package \texttt{ltmle} with GLM and a super learner (SL) library consisting of GLM, maltivariate adaptive regression spline with \texttt{earth} package, and \texttt{xgboost} for the simple synthetic data and the real world data \cite{ltmle_rpackage, superlearner_rpackage, earth_rpackage, xgboost_rpackage}. Confidence intervals for LTMLE was constructed based on its estimate of the efficient influence function. 

\begin{table}[ht]
\centering
\caption{Running time with complex synthetic data}
\label{tableTime2}
    \begin{tabular}{lccc}
    \toprule
     & \multicolumn{3}{c}{Time, sec} \\
    Model & $\tau = 10$ & $\tau = 20$ & $\tau = 30$ \\
    \midrule
    LTMLE (SL) & 271 & 958 & 2122 \\
    DeepACE & 53 & 54 & 133 \\
    Deep LTMLE & 38 & 39 & 116 \\
    \bottomrule
    \end{tabular}
\end{table}

As shown in Table \ref{tableTime2}, Deep LTMLE leverages GPU acceleration to achieve significantly faster processing times than LTMLE, presenting a substantial computational benefit for analyses involving extensive time horizons and high-dimensional time-dependent covariates. Code for experimens with synthetic data is available at https://github.com/shirakawatoru/dltmle-icml-2024.

\section{Discussion}

Our method assumes the sequential randomization and the positivity assumption on the intervention mechanism to identify the counterfactual outcome from observational data. However, to our surprise, in semi-synthetic data simulations, we found that when there is unmeasured confounding violating the sequential randomization assumption rely on, our method is very robust and could even provide robust inference. Furthermore, our proposed model does not currently address several complexities often found in real-world data, such as visiting processes, competing risks, and continuous time horizons. These challenges will be the focus of our future research efforts.

\section{Conclusion}

In this paper, we propose a variant of LTMLE that leverages the sequential learning capabilities of transformers. This approach enables simultaneous fitting of the entire LTMLE, allowing us to target the mean survival under dynamic interventions directly through weighting the loss function with cumulative inverse probabilities of intervention. The proposed method performs competitively with asymptotically efficient estimators in low-dimensional settings and exceeds the performance of existing models in high-dimensional scenarios. Scalability of our model to larger and longer datasets was implied. We applied our method to real world data and demonstrated a causal inference on the effect of sustained blood pressure management strategies on total mortality.

\section*{Acknowledgement}
This research is funded by NIH and Berkeley School of Public Health, Interdisciplinary Collaborative Research Grant. TS is supported by Fulbright scholarship program. The authors thank Dr. Ahmed Alaa at University of California, San Francisco and Berkeley for valuable discussions. The authors acknowledge the CIRCS investigators team for providing the real world data for experiments; Dr. Akihiko Kitamura at Yao City, Dr. Masahiko Kiyama at Osaka Center for Prevention of Cardiovascular Diseases, Dr. Takeo Okada at Osaka Center for Prevention of Cardiovascular Diseases, Dr. Yuji Shimizu at Osaka Center for Prevention of Cardiovascular Diseases, Dr. Hironori Imano at Kinki University, Dr. Tetsuya Ohira at Fukushima Prefeture Medical University, Dr. Kazumasa Yamagishi at Tsukuba University, and Dr. Isao Muraki at Osaka University.

\section*{Impact Statement}

The present study may have an impact on generating evidence for medical practice. Since the majority of medical evidence on the effect of treatments such as blood pressure control on the incidence of diseases or mortality are based on a single time point intervention study, the present study which takes multiple point intervention and long past patient histories into account might trigger the rethinking of the methodology for evidence synthesis and potentially affect the medical practice in the future after thorough statistically rigorous evaluation of the method and analyses of both synthetic and real world data.

\bibliography{main}

\begin{thebibliography}{41}
\providecommand{\natexlab}[1]{#1}
\providecommand{\url}[1]{\texttt{#1}}
\expandafter\ifx\csname urlstyle\endcsname\relax
  \providecommand{\doi}[1]{doi: #1}\else
  \providecommand{\doi}{doi: \begingroup \urlstyle{rm}\Url}\fi

\bibitem[Bang \& Robins(2005)Bang and Robins]{bangDoublyRobustEstimation2005}
Bang, H. and Robins, J.~M.
\newblock Doubly robust estimation in missing data and causal inference models.
\newblock \emph{Biometrics}, 61\penalty0 (4):\penalty0 962--973, 2005.

\bibitem[Bica et~al.(2020)Bica, Alaa, Jordon, and van~der Schaar]{Bica2020Estimating}
Bica, I., Alaa, A.~M., Jordon, J., and van~der Schaar, M.
\newblock Estimating counterfactual treatment outcomes over time through adversarially balanced representations.
\newblock In \emph{International Conference on Learning Representations}, 2020.

\bibitem[Bickel et~al.(1993)Bickel, Klaassen, Ritov, and Wellner]{bkrw}
Bickel, P., Klaassen, C., Ritov, Y., and Wellner, J.
\newblock \emph{Efficient and Adaptive Estimation for Semiparametric Models}.
\newblock Johns {{Hopkins}} Series in the Mathematical Sciences. {Springer New York}, 1993.
\newblock ISBN 978-0-387-98473-5.

\bibitem[Brown et~al.(2020)Brown, Mann, Ryder, Subbiah, Kaplan, Dhariwal, Neelakantan, Shyam, Sastry, Askell, et~al.]{brown2020language}
Brown, T., Mann, B., Ryder, N., Subbiah, M., Kaplan, J.~D., Dhariwal, P., Neelakantan, A., Shyam, P., Sastry, G., Askell, A., et~al.
\newblock Language models are few-shot learners.
\newblock \emph{Advances in neural information processing systems}, 33:\penalty0 1877--1901, 2020.

\bibitem[Chen et~al.(2022)Chen, He, Benesty, Khotilovich, Tang, Cho, Chen, Mitchell, Cano, Zhou, Li, Xie, Lin, Geng, Li, and Yuan]{xgboost_rpackage}
Chen, T., He, T., Benesty, M., Khotilovich, V., Tang, Y., Cho, H., Chen, K., Mitchell, R., Cano, I., Zhou, T., Li, M., Xie, J., Lin, M., Geng, Y., Li, Y., and Yuan, J.
\newblock \emph{xgboost: Extreme Gradient Boosting}, 2022.
\newblock URL \url{https://CRAN.R-project.org/package=xgboost}.
\newblock R package version 1.7.6.1.

\bibitem[Chernozhukov et~al.(2018)Chernozhukov, Chetverikov, Demirer, Duflo, Hansen, Newey, and Robins]{chernozhukovDML}
Chernozhukov, V., Chetverikov, D., Demirer, M., Duflo, E., Hansen, C., Newey, W., and Robins, J.
\newblock Double/debiased machine learning for treatment and structural parameters.
\newblock \emph{The Econometrics Journal}, 21\penalty0 (1):\penalty0 C1--C68, 2018.

\bibitem[Chernozhukov et~al.(2022)Chernozhukov, Escanciano, Ichimura, Newey, and Robins]{cherLocally}
Chernozhukov, V., Escanciano, J.~C., Ichimura, H., Newey, W.~K., and Robins, J.~M.
\newblock Locally robust semiparametric estimation.
\newblock \emph{Econometrica}, 90\penalty0 (4):\penalty0 1501--1535, 2022.

\bibitem[Dahabreh \& Bibbins-Domingo(2024)Dahabreh and Bibbins-Domingo]{dahabreh_causal_2024}
Dahabreh, I.~J. and Bibbins-Domingo, K.
\newblock Causal {Inference} {About} the {Effects} of {Interventions} {From} {Observational} {Studies} in {Medical} {Journals}.
\newblock \emph{JAMA}, May 2024.
\newblock ISSN 1538-3598.
\newblock \doi{10.1001/jama.2024.7741}.

\bibitem[Dang et~al.(2023)Dang, Gruber, Lee, Dahabreh, Stuart, Williamson, Wyss, Díaz, Ghosh, Kıcıman, Alemayehu, Hoffman, Vossen, Huml, Ravn, Kvist, Pratley, Shih, Pennello, Martin, Waddy, Barr, Akacha, Buse, van~der Laan, and Petersen]{dangCausalRoadmapGenerating2023}
Dang, L.~E., Gruber, S., Lee, H., Dahabreh, I.~J., Stuart, E.~A., Williamson, B.~D., Wyss, R., Díaz, I., Ghosh, D., Kıcıman, E., Alemayehu, D., Hoffman, K.~L., Vossen, C.~Y., Huml, R.~A., Ravn, H., Kvist, K., Pratley, R., Shih, M.-C., Pennello, G., Martin, D., Waddy, S.~P., Barr, C.~E., Akacha, M., Buse, J.~B., van~der Laan, M., and Petersen, M.
\newblock A causal roadmap for generating high-quality real-world evidence.
\newblock \emph{Journal of Clinical and Translational Science}, 7\penalty0 (1):\penalty0 e212, 2023.

\bibitem[Farajtabar et~al.(2018)Farajtabar, Chow, and Ghavamzadeh]{farajtabarMRDROPE}
Farajtabar, M., Chow, Y., and Ghavamzadeh, M.
\newblock More robust doubly robust off-policy evaluation.
\newblock In \emph{Proceedings of the 35th International Conference on Machine Learning}, volume~80 of \emph{Proceedings of Machine Learning Research}, pp.\  1447--1456. PMLR, 10--15 Jul 2018.

\bibitem[Frauen et~al.(2023)Frauen, Hatt, Melnychuk, and Feuerriegel]{Frauen_Hatt_Melnychuk_Feuerriegel_2023}
Frauen, D., Hatt, T., Melnychuk, V., and Feuerriegel, S.
\newblock Estimating average causal effects from patient trajectories.
\newblock \emph{Proceedings of the AAAI Conference on Artificial Intelligence}, 37\penalty0 (6):\penalty0 7586--7594, 2023.

\bibitem[Gruber \& van~der Laan(2010)Gruber and van~der Laan]{gruber_targeted_2010}
Gruber, S. and van~der Laan, M.~J.
\newblock A targeted maximum likelihood estimator of a causal effect on a bounded continuous outcome.
\newblock \emph{The International Journal of Biostatistics}, 6\penalty0 (1):\penalty0 Article 26, 2010.
\newblock ISSN 1557-4679.
\newblock \doi{10.2202/1557-4679.1260}.

\bibitem[Gruber \& van~der Laan(2012)Gruber and van~der Laan]{tmle_plug_in}
Gruber, S. and van~der Laan, M.~J.
\newblock Targeted minimum loss based estimation of a causal effect on an outcome with known conditional bounds.
\newblock \emph{The international journal of biostatistics}, 8\penalty0 (1):\penalty0 21--21, 2012.
\newblock ISSN 1557-4679.

\bibitem[Hernán et~al.(2022)Hernán, Wang, and Leaf]{hernan_target_2022}
Hernán, M.~A., Wang, W., and Leaf, D.~E.
\newblock Target {Trial} {Emulation}: {A} {Framework} for {Causal} {Inference} {From} {Observational} {Data}.
\newblock \emph{JAMA}, 328\penalty0 (24):\penalty0 2446--2447, December 2022.
\newblock ISSN 1538-3598.
\newblock \doi{10.1001/jama.2022.21383}.

\bibitem[Jiang \& Li(2016)Jiang and Li]{jiangDROPE}
Jiang, N. and Li, L.
\newblock Doubly robust off-policy value evaluation for reinforcement learning.
\newblock In \emph{Proceedings of The 33rd International Conference on Machine Learning}, volume~48 of \emph{Proceedings of Machine Learning Research}, pp.\  652--661, New York, New York, USA, 20--22 Jun 2016. PMLR.

\bibitem[Kallus \& Uehara(2020)Kallus and Uehara]{ueharaDRL}
Kallus, N. and Uehara, M.
\newblock Double reinforcement learning for efficient and robust off-policy evaluation.
\newblock In \emph{Proceedings of the 37th International Conference on Machine Learning}, volume 119 of \emph{Proceedings of Machine Learning Research}, pp.\  5078--5088. PMLR, 13--18 Jul 2020.

\bibitem[Kennedy(2022)]{kennedySemiparametricDoublyRobust2022}
Kennedy, E.~H.
\newblock Semiparametric doubly robust targeted double machine learning: A review.
\newblock \emph{arXiv preprint arXiv:2203.06469}, 2022.

\bibitem[Klaassen(1987)]{klaassen1987}
Klaassen, C. A.~J.
\newblock Consistent estimation of the influence function of locally asymptotically linear estimators.
\newblock \emph{The Annals of Statistics}, 15\penalty0 (4):\penalty0 1548--1562, 1987.

\bibitem[Lendle et~al.(2017)Lendle, Schwab, Petersen, and {van der Laan}]{ltmle_rpackage}
Lendle, S.~D., Schwab, J., Petersen, M.~L., and {van der Laan}, M.~J.
\newblock {ltmle}: An {R} package implementing targeted minimum loss-based estimation for longitudinal data.
\newblock \emph{Journal of Statistical Software}, 81\penalty0 (1):\penalty0 1--21, 2017.
\newblock \doi{10.18637/jss.v081.i01}.

\bibitem[Levine et~al.(2020)Levine, Kumar, Tucker, and Fu]{levineORL}
Levine, S., Kumar, A., Tucker, G., and Fu, J.
\newblock Offline {{Reinforcement Learning}}: {{Tutorial}}, {{Review}}, and {{Perspectives}} on {{Open Problems}}.
\newblock \emph{arXiv preprint arXiv:2005.01643}, 2020.

\bibitem[Li et~al.(2021)Li, Hu, Lu, Utsumi, Chakraborty, Sow, Madan, Li, Ghalwash, Shahn, and Lehman]{liGNetDeepLearning2020a}
Li, R., Hu, S., Lu, M., Utsumi, Y., Chakraborty, P., Sow, D.~M., Madan, P., Li, J., Ghalwash, M., Shahn, Z., and Lehman, L.-w.
\newblock G-net: A recurrent network approach to {{G-computation}} for counterfactual prediction under a dynamic treatment regime.
\newblock In \emph{Proceedings of Machine Learning for Health}, volume 158 of \emph{Proceedings of Machine Learning Research}, pp.\  282--299. {PMLR}, 2021.

\bibitem[Melnychuk et~al.(2022)Melnychuk, Frauen, and Feuerriegel]{pmlr-v162-melnychuk22a}
Melnychuk, V., Frauen, D., and Feuerriegel, S.
\newblock Causal transformer for estimating counterfactual outcomes.
\newblock In Chaudhuri, K., Jegelka, S., Song, L., Szepesvari, C., Niu, G., and Sabato, S. (eds.), \emph{Proceedings of the 39th International Conference on Machine Learning}, volume 162 of \emph{Proceedings of Machine Learning Research}, pp.\  15293--15329. {PMLR}, 2022.

\bibitem[Milborrow(2023)]{earth_rpackage}
Milborrow, S.
\newblock \emph{earth: Multivariate Adaptive Regression Splines}, 2023.
\newblock URL \url{https://CRAN.R-project.org/package=earth}.
\newblock R package version 5.3.2.

\bibitem[Mnih et~al.(2013)Mnih, Kavukcuoglu, Silver, Graves, Antonoglou, Wierstra, and Riedmiller]{mnih2013playing}
Mnih, V., Kavukcuoglu, K., Silver, D., Graves, A., Antonoglou, I., Wierstra, D., and Riedmiller, M.
\newblock Playing atari with deep reinforcement learning.
\newblock \emph{arXiv preprint arXiv:1312.5602}, 2013.

\bibitem[Narita et~al.(2021)Narita, Yasui, and Yata]{naritaDOPE}
Narita, Y., Yasui, S., and Yata, K.
\newblock Debiased off-policy evaluation for recommendation systems.
\newblock In \emph{Proceedings of the 15th ACM Conference on Recommender Systems}, RecSys '21, pp.\  372–379, New York, NY, USA, 2021. Association for Computing Machinery.
\newblock ISBN 9781450384582.

\bibitem[Petersen \& van~der Laan(2014)Petersen and van~der Laan]{petersenCausalModelsLearning2014}
Petersen, M.~L. and van~der Laan, M.~J.
\newblock Causal models and learning from data: Integrating causal modeling and statistical estimation.
\newblock \emph{Epidemiology (Cambridge, Mass.)}, 25\penalty0 (3):\penalty0 418--426, 2014.

\bibitem[Polley et~al.(2021)Polley, LeDell, Kennedy, and {van der Laan}]{superlearner_rpackage}
Polley, E., LeDell, E., Kennedy, C., and {van der Laan}, M.
\newblock \emph{SuperLearner: Super Learner Prediction}, 2021.
\newblock URL \url{https://CRAN.R-project.org/package=SuperLearner}.
\newblock R package version 2.0-28.1.

\bibitem[Robins(1986)]{robins1986new}
Robins, J.
\newblock A new approach to causal inference in mortality studies with a sustained exposure period—application to control of the healthy worker survivor effect.
\newblock \emph{Mathematical modelling}, 7\penalty0 (9-12):\penalty0 1393--1512, 1986.

\bibitem[Robins et~al.(1994)Robins, Rotnitzky, and Zhao]{robins1994}
Robins, J.~M., Rotnitzky, A., and Zhao, L.~P.
\newblock Estimation of {{Regression Coefficients When Some Regressors Are Not Always Observed}}.
\newblock \emph{Journal of the American Statistical Association}, 89\penalty0 (427):\penalty0 846--866, 1994.

\bibitem[Salerno \& Li(2023)Salerno and Li]{salerno2023high}
Salerno, S. and Li, Y.
\newblock High-dimensional survival analysis: Methods and applications.
\newblock \emph{Annual review of statistics and its application}, 10:\penalty0 25--49, 2023.

\bibitem[Sutton(1988)]{sutton1988learning}
Sutton, R.~S.
\newblock Learning to predict by the methods of temporal differences.
\newblock \emph{Machine learning}, 3:\penalty0 9--44, 1988.

\bibitem[Tran et~al.(2019)Tran, Yiannoutsos, Wools-Kaloustian, Siika, van~der Laan, and Petersen]{tran_double_2019}
Tran, L., Yiannoutsos, C., Wools-Kaloustian, K., Siika, A., van~der Laan, M., and Petersen, M.
\newblock Double robust efficient estimators of longitudinal treatment effects: Comparative performance in simulations and a case study.
\newblock \emph{The International Journal of Biostatistics}, 15\penalty0 (2), 2019.
\newblock ISSN 1557-4679.
\newblock \doi{10.1515/ijb-2017-0054}.

\bibitem[van~der Laan \& Rubin(2006)van~der Laan and Rubin]{laanTMLE2006}
van~der Laan, M. and Rubin, D.
\newblock Targeted {{Maximum Likelihood Learning}}.
\newblock \emph{The International Journal of Biostatistics}, 2\penalty0 (1), 2006.

\bibitem[van~der Laan \& Gruber(2012)van~der Laan and Gruber]{laanTargetedMinimumLoss2012}
van~der Laan, M.~J. and Gruber, S.
\newblock Targeted {{Minimum Loss Based Estimation}} of {{Causal Effects}} of {{Multiple Time Point Interventions}}.
\newblock \emph{The International Journal of Biostatistics}, 8\penalty0 (1), 2012.

\bibitem[van~der Laan \& Robins(2003)van~der Laan and Robins]{van2012unified}
van~der Laan, M.~J. and Robins, J.
\newblock \emph{Unified Methods for Censored Longitudinal Data and Causality}.
\newblock Springer Series in Statistics. {Springer New York}, 2003.
\newblock ISBN 978-0-387-21700-0.

\bibitem[van~der Laan \& Rose(2011)van~der Laan and Rose]{laanTLBook2011}
van~der Laan, M.~J. and Rose, S.
\newblock \emph{Targeted {{Learning}}: {{Causal Inference}} for {{Observational}} and {{Experimental Data}}}.
\newblock Springer {{Series}} in {{Statistics}}. {Springer}, 2011.
\newblock ISBN 978-1-4419-9781-4 978-1-4419-9782-1.

\bibitem[van~der Laan \& Rose(2018)van~der Laan and Rose]{van2018targeted}
van~der Laan, M.~J. and Rose, S.
\newblock \emph{Targeted Learning in Data Science: {{Causal}} Inference for Complex Longitudinal Studies}.
\newblock Springer Series in Statistics. {Springer International Publishing}, 2018.

\bibitem[van~der Laan et~al.(2007)van~der Laan, Polley, and Hubbard]{superlearner}
van~der Laan, M.~J., Polley, E.~C., and Hubbard, A.~E.
\newblock Super learner.
\newblock \emph{Statistical Applications in Genetics and Molecular Biology}, 6\penalty0 (1):\penalty0 1309--1309, 2007.
\newblock ISSN 1544-6115.

\bibitem[Vaswani et~al.(2017)Vaswani, Shazeer, Parmar, Uszkoreit, Jones, Gomez, Kaiser, and Polosukhin]{vaswani2017a}
Vaswani, A., Shazeer, N., Parmar, N., Uszkoreit, J., Jones, L., Gomez, A.~N., Kaiser, L., and Polosukhin, I.
\newblock Attention is all you need.
\newblock In \emph{Advances in Neural Information Processing Systems}, volume~30. {Curran Associates, Inc.}, 2017.

\bibitem[Wyss et~al.(2022)Wyss, Yanover, El‐Hay, Bennett, Platt, Zullo, Sari, Wen, Ye, Yuan, Gokhale, Patorno, and Lin]{wyss_machine_learning_2022}
Wyss, R., Yanover, C., El‐Hay, T., Bennett, D., Platt, R.~W., Zullo, A.~R., Sari, G., Wen, X., Ye, Y., Yuan, H., Gokhale, M., Patorno, E., and Lin, K.~J.
\newblock Machine learning for improving high‐dimensional proxy confounder adjustment in healthcare database studies: An overview of the current literature.
\newblock \emph{Pharmacoepidemiology and drug safety}, 31\penalty0 (9):\penalty0 932--943, 2022.
\newblock ISSN 1053-8569.

\bibitem[Yamagishi et~al.(2019)Yamagishi, Muraki, Kubota, Hayama-Terada, Imano, Cui, Umesawa, Shimizu, Sankai, Okada, Sato, Kitamura, Kiyama, and Iso]{yamagishi2019}
Yamagishi, K., Muraki, I., Kubota, Y., Hayama-Terada, M., Imano, H., Cui, R., Umesawa, M., Shimizu, Y., Sankai, T., Okada, T., Sato, S., Kitamura, A., Kiyama, M., and Iso, H.
\newblock The {{Circulatory Risk}} in {{Communities Study}} ({{CIRCS}}): {{A Long-Term Epidemiological Study}} for {{Lifestyle-Related Disease Among Japanese Men}} and {{Women Living}} in {{Communities}}.
\newblock \emph{Journal of Epidemiology}, 29\penalty0 (3):\penalty0 83--91, 2019.

\end{thebibliography}
\bibliographystyle{icml2024}

\newpage
\appendix
\onecolumn

\section{Notation}

Here we list notations used in the article.

\begin{center}
\begin{tabular}{ll}
     $O$ & Observed variables $O=(W=W_0, L_1, A_1, Y_1, L_2, A_2, Y_2, \ldots, L_\tau, A_\tau, Y_\tau)$ \\
     $\tau$ & Maximum length of time-horizon \\
     $T$ & Stopping time \\
     $W$ & Baseline covariates \\
     $L_t$ & Time-dependent covariates (states) \\
     $A_t$ & Time-dependent treatments (controls) \\
     $Y_t$ & Outcomes. In survival case, binary failure indicator defined as $Y_t = \mathbbm{1}\{T\le t\}$ \\
     $Y$ & Outcome at the end of the trajectory: $Y=Y_T$ \\
     $pa(\bullet)$ & Parent nodes of a node $\bullet$. For example, $pa(L_t)=(W,L_{1:t-1}, A_{1:t-1}, Y_{1:t-1})$ \\
     $P_0$ & The true distribution of the observed variable\\
     $\hat{P}_n$ & Estimator of $P_0$ \\
     $\pi$ & Propensity scores $\pi=[\pi_t]_1^\tau$ with $\pi(da_t|pa(a_t))=\mathbb{P}(da_t|pa(a_t))$ \\
     $g$ & User-specified treatment policies $g=[g_t]_{t=1}^\tau$ \\
     $\psi$ & Target functional $\psi(P)=\mathbb{E}_gY$ \\
     $\psi_0$ & True parameter $\psi_0 = \psi(P_0)$ \\
     $\hat{\psi}_n$ & Estimator of $\psi_0$ \\
     $\hat{\sigma}_n$ & Estimator of the standard error of the estimator $\hat{\psi}_n$ \\
     $Q$ & State-action value functions $Q=[Q_t]_{t=1}^\tau$ \\
     $Q_t$ & $=Q_t(pa(Y_t))=\mathbb{E}_g[Y|pa(Y_t)]$ \\
     $V$ & Value functions $V=[V_t]_{t=1}^{\tau+1}$ \\
     $V_t$ & $=V_t(pa(A_t))=\mathbb{E}_g[Y|pa(A_t)]$ for $t=1,\ldots,\tau$ and $V_{\tau+1}=Y_\tau$\\
     $I_t$ & Clever covariates (importance weights) $I_t=\prod_{s=1}^t dg_s/d\pi_s(O)$ \\
     $D^\star$ & Efficient influence function of $\psi$: $D^\star(Q, V; \pi) = V_1 - \psi_0 + \sum_{t=1}^\tau (V_{t+1} - Q_t)$ \\
     $Q_\varepsilon$ & Local least favorable submodel $Q_\varepsilon=[Q_{t,\varepsilon}]_{t=1}^\tau$ \\
     $Q_t$ & $\logit Q_{t,\varepsilon} = \logit Q_t +\varepsilon$ \\
     $V_\varepsilon$ & Local least favorable submodel $V_\varepsilon=[V_{t,\varepsilon}]_{t=1}^{\tau+1}$  \\
     $V_t$ & $\logit V_{t,\varepsilon} = \logit V_t +\varepsilon$ for $t=1,\ldots,\tau$ and $V_{\tau+1,\varepsilon}=V_{\tau+1}$ \\
     $\mathcal{L}(Q, V)$ & Loss function for temporal difference learning \\
     $\mathcal{L^\star}$ & Loss function for targeting \\
     $\alpha$ & Weight for the propensity loss (hyperparameter) \\
     $Pf$ & Mean of a function $f$ under the distribution $P$: $Pf=\int fdP$ \\
     $H(\bullet_t)$ & Embedding of a node $\bullet_t$\\
     $h_\bullet(\bullet_t)$ & Type embedding of a node $\bullet_t$\\
     $E_t$ & Positional encoding at time $t$ \\
     $f_X$ & production function of a node $X$
     
\end{tabular}
    
\end{center}

\section{Proof}

\label{sec:proof}

\begin{proof}[Proof of Theorem \ref{thm:plf}]
    A direct calculation shows
    \begin{align*}
        \partial_\varepsilon \mathcal{L}^\star_t(Q_{t,\varepsilon}, V_{t+1,\varepsilon^\star})
        &= I_t(\pi)[\partial_\varepsilon Q_{t,\varepsilon}] [\partial_{Q_{t,\varepsilon}}\mathcal{L}_{\mathrm{bce}}(Q_{t,\varepsilon}, V_{t+1,\varepsilon^\star})] = I_{t}(\pi)(V_{t+1,\varepsilon^\star} - Q_{t,\varepsilon}).
    \end{align*}
    Substitution of $\varepsilon^\star$ to $\varepsilon$ yields the $t$-th component of the efficient influence function \eqref{eq:eif} at $(Q_{\varepsilon^\star}, V_{\varepsilon^\star}, \pi)$.
\end{proof}

\section{Review of TMLE}

\label{sec:tmle}

\subsection{Structural Causal Model}

\label{sec:scm}

We assume each node depends on the all previous nodes in the trajectory, that is, we do not assume the Markovian property. And each node $X$ is produced from the parent nodes $pa(X)$ and independent noise random variables $U_X$ by a measurable function $f_X$: $X=f_X(pa(X),U_X)$. This production function $f_X$ induces a conditional distribution $P_{X|H}$ of $X$ given $H=pa(X)$ by pushing forward the distribution of noise variable: $P_{X|H}(A|h)=(P_{U_X}\circ f_X^{-1}(h, \cdot))(A)$ for all measurable $A\subset\mathcal{X}$, where $\mathcal{X}$ is a domain of random vector $X$. Starting from nodes without parents including noise nodes and their distributions, production functions and their causal structure, which can be described by a directed acyclic graph over the ovservables, generate the joint distribution of the observed random variables. With our particular data in longitudinal setting, we define the propensity score $\pi_t=P_{A_t|pa(A_t)}$, where $pa(A_t)=pa(A_t)$ is the patient history before the node $A_t$. We use the same symbol for the production function if the treatment assignment is deterministics, that is, there is no noise variable in generating the treatment node: $d\pi_t(A_t|pa(A_t))=1$ if $A_t=a_t$ for some specific $a_t$.

\subsection{Causal Target Parameter and Identification}

Our target parameter is the counterfactual mean of the final outcome $Y$ under the user-specified dynamic treatment policy $g=(g_t)$. This is the mean of counterfactual outcome which is produced by replacing $\pi_t$ with $g_t$ in the structural causal model:
\begin{equation}
    \psi^g(P) = EY^g.
\end{equation}
To identify this causal target paratmer from observatoinal data, we assume the following conditions of the positiviy:
\begin{equation}
    g \ll \pi,    
\end{equation}
and the sequential randomization:
\begin{equation}
    Y^g \perp A_t \mid pa(A_t) \text{ for } t=1,\ldots,\tau.    
\end{equation}
Note that the consistency $Y=Y^{\pi}$ usually stated in the causal inference literature is a consequence of the definition of counterfactual outcome in our structural causal model. Under these identifiability conditions, this parameter is identified through g-formula that is the mean of $Y$ under the counterfactual distribution which is given by replacing distributions $\pi_t$ with $g_t$:
\begin{equation}
    EY^g = E_gY.
\end{equation}
Then the problem reduced to the estimation of the statistical parameter:
\begin{equation}
    \psi(P) = E_gY.
\end{equation}

\subsection{TMLE} \label{TMLE-appendix}
Bias correction by TMLE is based on the following first order approximation of the target functional around the true distribution $P_0$ \cite{laanTMLE2006,laanTLBook2011,kennedySemiparametricDoublyRobust2022}:
\begin{equation}
    \label{eq:von-mises}
    \psi(\hat{P}_n) - \psi(P_0) = -P_0D^\star(\hat{P}_n) + R_2(\hat{P}_n,P_0),
\end{equation}
where $D^\star$ is called influence function and $R_2$ is the second order remainder. This equation is the infinite dimensional extension of Taylor expansion. 

The right hand side of this equation can be further written as:
\begin{equation}
    -P_nD^\star(\hat{P}_n) + (P_n-P_0)\big[D^\star(\hat{P}_n) -D^\star(P_0)\big]  + (P_n-P_0)D^\star(P_0)+ R_2(\hat{P}_n,P_0),
\end{equation}
whose second term called empirical process term converges to zero in the rate of square root of $n$ if $D^\star(\hat{P}_n), D^\star(P_0)$ belong to the Donsker class and the $L^2(P_0)$-distance between $D^\star(\hat{P}_n)$ and $D^\star(P_0)$ converges to zero in probability. Given a good initial fit $\hat{P}_n$ of $P_0$, above conditions are usually satisfied and, in addition, we require $R_2(\hat{P}_n,P_0)=o_{P_0}(n^{-1/2})$ which we will look at closely in Section \ref{sec:dr}. By further using the fact about the influence function that $P_0D^\star(P_0)=0$, the right hand side reduced to
\begin{equation}
    -P_nD^\star(\hat{P}_n) + P_nD^\star(P_0)+ o_{P_0}(n^{-1/2}).
\end{equation}

Now, the idea is to find $\hat{P}_n^\star$ in the close neighborhood of $\hat{P}_n$ that solves the empirical analog of the first term:
\begin{equation}
    \label{eq:eee}
    P_nD^\star(\hat{P}_n^\star) = 0.
\end{equation}

By doing so, using similar arguments as above for $\hat{P}_n^\star$ instead of $\hat{P}_n$, we have the following.
\begin{equation}
    \psi(\hat{P}_n^\star) - \psi(P_0) = P_nD^\star(P_0) + o_{P_0}(n^{-1/2}).
\end{equation}
Thus, our estimator $\psi(\hat{P}_n^\star)$ is a plug in estimator and attains the efficiency bound among the asymptotically linear and regular estimators.

\subsection{Efficient Influence Curve}

Then the efficient influence function of our target parameter \eqref{eq:g-functional} is computed as follows \cite{laanTargetedMinimumLoss2012}
\begin{equation}
    \label{eq:eif}
    \begin{aligned}
        D^\star(P)&=\sum_{t=1}^{\tau}D^\star_t(P)
        = (V_1 - \psi_0) +\sum_{t=1}^{\tau}  I_{t}(V_{t+1} - Q_t)
        =(V_1 - \psi_0)+\sum_{t=1}^{T}I_{t}(V_{t+1} - Q_t).
    \end{aligned}
\end{equation}
The last equality follows since $V_{t+1}=Q_t$ if $t>T$.

\subsection{Second-order Remainder and Double Robustness}

\label{sec:dr}

In this subsection, we derive the exact remainder (Theorem \ref{thm:R2}) and see its double robustness (Theorem \ref{thm:dr}).

\begin{theorem}
    \label{thm:R2}
    The second-order remainder of the target functional \eqref{eq:g-functional} is given by:
    \begin{equation}
        \label{eq:R2}
        R_2(\hat{P},P)=\sum_{t=1}^\tau\int_{\mathcal{O}_{1:t}}\frac{dg_{1:t}}{d\hat{\pi}_{1:t}}\Big[\int V_{t+1}(\hat{\eta})d\eta_{t+1}-Q_{t}(\hat{\eta})\Big]d(\pi_{1:t}-\hat{\pi}_{1:t})d\eta_{1:t}.
    \end{equation}
\end{theorem}

\begin{proof}
    First, we could consider the TDHD (Figure \ref{fig:architecture}) implicitly learns conditional distributions $\eta_{t}=P_{Y_{t-1},L_{t}|pa(Y_{t-1})}$ for $t=2,\ldots,\tau$ and outputs 
\begin{align}
    \hat{Q}_t&=Q_t(\hat{\eta})=(g_{t+1:\tau}\otimes\hat{\eta}_{t+1:\tau+1})[Y] = \int_{\mathcal{O}_{t+1:\tau}} Yd(g_{t+1:\tau}\otimes\hat{\eta}_{t+1:\tau+1})\\ 
    \hat{V}_t&=V_t(\hat{\eta})=(g_{t:\tau}\otimes\hat{\eta}_{t+1:\tau+1})[Y] =\int_{\mathcal{O}_{t+1:\tau}} Yd(g_{t:\tau}\otimes\hat{\eta}_{t+1:\tau+1})
\end{align}
for $t=1,\ldots,\tau$. Here we used a notation $\mu\nu f=(\mu\otimes\nu)[f]=\int fd(\mu\otimes\nu)$ for disintegrations of measures $\mu\otimes\nu$ and functions $f$, identifying the measure $\mu\otimes\nu$ with a linear functional from some functional class $\mathcal{F}$ to the real numbers. Let $\hat{\eta}_1$ be the empirical distribution of $W$ and $L_1$, $P=\eta\otimes\pi$ be the true and $\hat{P}=\hat{\eta} \otimes\hat{\pi}$ be the initial estimate of the joint density of the observables. By definition \eqref{eq:von-mises}, the second-order remainder is
\begin{equation}
    \label{eq:remainder}
\begin{aligned}
    R_2(\hat{P},P)&=\psi(\hat{P})-\psi(P)+PD^\star(\hat{\eta}\otimes\hat{\pi}) \\
    &= \Big[g(\hat\eta-\eta)+g(\eta_1\hat{\eta}_{2:\tau+1}-\hat{\eta}_{1:\tau+1})+\sum_{t=1}^\tau\frac{dg_{1:t}}{d\hat{\pi}_{1:t}}\pi_{1:t}g_{t+1:\tau}\eta_{1:t}(\eta_{t+1}-\hat{\eta}_{t+1})
    \hat{\eta}_{t+2:\tau+1}\Big][Y] \\
    &= \Big[g\sum_{t=1}^\tau\big(\eta_{1:t}\hat{\eta}_{t+1:\tau+1}-\eta_{1:t+1}\hat{\eta}_{t+2:\tau+1}\big)+\sum_{t=1}^\tau\frac{dg_{1:t}}{d\hat{\pi}_{1:t}}\pi_{1:t}g_{t+1:\tau}\eta_{1:t}(\eta_{t+1}-\hat{\eta}_{t+1})
    \hat{\eta}_{t+2:\tau+1}\Big][Y] \\
    & = \Big[\sum_{t=1}^\tau\frac{dg_{1:t}}{d\hat{\pi}_{1:t}}(\pi_{1:t}-\hat{\pi}_{1:t})g_{t+1:\tau}\eta_{1:t}(\eta_{t+1}-\hat{\eta}_{t+1})
    \hat{\eta}_{t+2:\tau+1}\Big][Y] \\
    &=\sum_{t=1}^\tau\int_{\mathcal{O}_{1:t}}\frac{dg_{1:t}}{d\hat{\pi}_{1:t}}\Big[\int V_{t+1}(\hat{\eta})d\eta_{t+1}-Q_{t}(\hat{\eta})\Big]d(\pi_{1:t}-\hat{\pi}_{1:t})d\eta_{1:t}.
\end{aligned}
\end{equation}
\end{proof}

\begin{theorem}
    \label{thm:dr}
    Suppose $\delta=\sup_{t,o} dg_{1:t}/d\hat{\pi}_{1:t}(o)$ exists and $\big\Vert \hat{Q}_{t}-Q_{t}
    \big\Vert_{\mathbb{L}^2(\nu_{1:t}\otimes\eta_{1:t})}
    \big\Vert \hat{e}_s-e_s\big\Vert_{\mathbb{L}^2(\nu_{1:s}\otimes\eta_{1:s})} = o_P(n^{-1/2})$ for all $t$ and $s\le t$, then $R_2(\hat{P},P)=o_P(n^{-1/2})$.
\end{theorem}

\begin{proof}
Noting that 
\begin{equation}
    \int V_{t+1}(\eta)d\eta_{t+1} - Q_t(\eta) = 0,
\end{equation}
since the first term in the left side is the definition of $Q_t$, we have
\begin{equation}
    \label{eq:zero_residual}
    \sum_{t=1}^\tau\int_{\mathcal{O}_{1:t}}\frac{dg_{1:t}}{d\hat{\pi}_{1:t}}\Big[\int V_{t+1}(\eta)d\eta_{t+1}-Q_t(\eta)\Big]d(\pi_{1:t}-\hat{\pi}_{1:t})d\eta_{1:t}=0.
\end{equation}
Let $e_t=d\pi_t/d\nu_t$ and $\hat{e}_t=d\hat{\pi}_t/d\nu_t$ be the propensity scores with respect to the counting measures $\nu_t$ over the action space. By subtracting \eqref{eq:zero_residual} from the second-order remainder \eqref{eq:R2}, we have
\begin{align}
    R_2(\hat{P},P)=\sum_{t=1}^\tau\int \frac{dg_{1:t}}{d\hat{\pi}_{1:t}}\big(V_{t+1}(\hat{\eta})-V_{t+1}(\eta)\big)\Big[\sum_{s=1}^te_{1:s-1}(e_s-\hat{e}_s)\hat{e}_{s+1:t}\Big]d\nu_{1:t}d\eta_{1:t+1} \\
    -\sum_{t=1}^\tau\int \frac{dg_{1:t}}{d\hat{\pi}_{1:t}}\big(Q_t(\hat{\eta})-Q_t(\eta)\big)\Big[\sum_{s=1}^te_{1:s-1}(e_s-\hat{e}_s)\hat{e}_{s+1:t}\Big]d\nu_{1:t}d\eta_{1:t}.
\end{align}
Applying the Cauchy-Schwartz inequality, we have 
\begin{align}
    \big|R_2(\hat{P},P)\big| \le
    2\delta \sum_{t=1}^\tau \sum_{s=1}^t
    \big\Vert Q_{t}(\hat{\eta})-Q_{t}(\eta)
    \big\Vert_{\mathbb{L}^2(\nu_{1:t}\otimes\eta_{1:t})}
    \big\Vert e_s-\hat{e}_s\big\Vert_{\mathbb{L}^2(\nu_{1:s}\otimes\eta_{1:s})}.
\end{align}
\end{proof}



\section{Convergence of Temporal Difference Learning}

\label{sec:convergence of td}

First, consider a flexible model $Q_\theta$ and corresponding $V_{t,\theta}=\mathbb{E}_{A_t\sim g_t}Q_{t,\theta}$. Initiate $\theta^0$ and then iteratively update by $\theta^{k+1}=\argmin_{\theta}P\mathcal{L}(Q_\theta,V_{\theta^k})$ for $k=2,\ldots,$ till convergence. Our proof below shows that if we use a variation independent parameter space for each $Q_{t, \theta}$ and the parameter spaces contain the true $Q_{t,P}$, then in $K+2$-steps this algorithm will have converged to the true solution $Q_{P}$.

Ignoring the parameterization, but just thinking in terms of optimizing over parameter spaces, this algorithm corresponds with: initiate $V^0$, and then for $k=0, \ldots$, compute $Q^{k+1}=\arg \min _{Q} P \mathcal{L}\left(Q, V^k\right)$ and set $V^{k+1}$ as the one implied by the intervention $g$ and $Q^{k+1}$; and set $k=k+1$.

Firstly, we claim that in a nonparametric model the $t$-specific parameters $Q_t$ are variation independent across $t$. Consider a given $V$ (misspecified). This implies a parameter space $\left\{\mathbb{E}_{L_t\mid pa(L_t)\sim\mu_t(\cdot\mid pa(L_t))}V_t:\mu_t\right\}$ for the regressions $Q$. The parameter space of the free parameter $\mu_t$ is even larger than the parameter space of functions of $pa(L_t)$. Therefore this appears indeed a reasonable condition. Then we can state that the parameter space over which we optimize at step $k$ of the algorithm is the cartesian product of the parameter spaces $\mathcal{Q}_t$ for $Q_t$ across $t=\tau, \ldots, 1$. Consider the $k=1$-step of the algorithm in which the outcomes are $V^0$ and we optimize over all the $Q \in \prod_{t=\tau}^1 \mathcal{Q}_t$. Then, $Q^1$ is the minimizer of $Q \rightarrow P \mathcal{L}(Q, V^0)$. That means that the derivative w.r.t. $\varepsilon_t$ along a path $Q_t^1+\varepsilon_t h_t$ through $Q_t^1$ in any direction $h_t$ at $\epsilon_t=0$ should be equal to zero, across all $t=\tau, \ldots, 1$. Thus, at $\varepsilon=0$, we have
$$
\frac{d}{d\varepsilon} \sum_{t=1}^\tau \mathcal{L}_t(Q_t^1+\varepsilon_t h_t\mid V^0_{t+1}) = 0
$$
Consider the derivative w.r.t. $\varepsilon_Y$. This yields the score equation $P h_{Q_\tau}(V_{\tau+1}-Q^1_\tau)=0$ for all $h_{Q_\tau}$. This implies that $Q^1_\tau = Q_{\tau,P}$. The others are some optimizer. Now, we go to step $k=2$. We now know that $V^1_\tau=V_{\tau, P}$ due to $Q^1_\tau=Q_{\tau, P}$. Therefore, at the next step, due to the derivative w.r.t. $\varepsilon_{\tau-1}$, it follows that $Q^2_{\tau-1}=Q_{\tau-1, P}$, while it again $Q_\tau^2=Q_\tau^1=Q_{\tau, P}$. Then, at step $k=3$, we also obtain $Q_{\tau-2}^3=Q_{\tau-2, P}$. In this manner, it follows that after $K+2$ steps we have $Q^{K+2}=Q_{P}$.

\section{Evaluation of $V$}

To evaluate $V_t$ using a fit $\hat{Q}_t$, we integrate $A_t$ out under the counterfactual intervention $g_t$. Since we assume deterministic intervention $A^g_t = g(pa(A_t))$, we computed it as $\hat{V}_t = \hat{Q}_t(L_{1:t},a_{1:t})$ where $a_s$ is defined  by $a_s=g_s(L_{1:s},A_{1:s-1})$ for $s=1,\ldots,t$. This is based on the almost sure equality $\mathbb{E}_g[Y|L_{1:t},A_{1:t-1}] = \mathbb{E}_g[Y|L_{1:t},a_{1:t-1}]$ as random variables. To see this, note that for general random variables $X$ and $Y$ the conditional expectation $\mathbb{E}[Y|X=x]$ can take any value whenever $\mathbb{P}(X=x)=0$. An alternative approatch is computing $\hat{V}$ values as $\hat{V}_t = \hat{Q}_t(L_{1:t},A_{1:t-1},a_t)$ and both are almost surely equivalent. Here we supressed $Y$ values from the history assuming all those values were zero without loss of generality.

\section{Hyperparameter Tuning}
We selected hyperparameters shown in Table \ref{tab:hparams} which optimized the empiricall loss $\mathcal{L}^Q + \mathcal{L}^e$ in the validation set which is the 30\% of the entire dataset. The parameter $\alpha$ and $\beta$ for censoring mechanism balances the learning rate of $\pi$-parts and $Q$-parts because the complexity of $\pi$-parts would be simpler than $Q$-parts which involves prediction in the long-range.

\begin{table}[ht]
    \centering
    \caption{Selected hyperparameters.}
    \label{tab:hparams}
\resizebox{\textwidth}{!}{
    \begin{tabular}{lccccccccccccc}
    \toprule
    Data  & \multicolumn{6}{c}{Simple Synthetic Data} & \multicolumn{6}{c}{Complex Synthetic Data} & Real World \\
    \midrule
    Model & \multicolumn{3}{c}{Deep LTMLE} & \multicolumn{3}{c}{DeepACE} & \multicolumn{3}{c}{Deep LTMLE} & \multicolumn{3}{c}{DeepACE} & Deep LTMLE \\
    \midrule
    $\tau$ & 10 & 20 & 30 & 10 & 20 & 30 & 10 & 20 & 30 & 10 & 20 & 30 & 30 \\
    \midrule
    Embedding dimension & 16 & 16 & 32 & --- & --- & --- & 16 & 32 & 32 & --- & --- & --- & 32 \\
    Dropout rate & 0.1 & 0.2 & 0 & 0.1 & 0.1 & 0.2 & 0 & 0 & 0 & 0.3 & 0.3 & 0.3 & 0 \\
    Hidden size & 32 & 32 & 16 & 16 & 32 & 32 & 64 & 32 & 16 & 16 & 16 & 16 & 16 \\
    Number of Layers & 8 & 4 & 8 & 4 & 4 & 4 & 4 & 4 & 4 & 8 & 8 & 4 & 8 \\
    Number of heads & 4 & 4 & 4 & --- & --- & --- & 8 & 8 & 8 & --- & --- & --- & 4 \\
    Learning rate & 1e-3 & 5e-4 & 1e-4 & 5e-3 & 5e-4 & 1e-3 & 1e-3 & 5e-4 & 1e-4 & 5e-4 & 5e-4 & 5e-4 & 5e-4 \\
    $\alpha$ & 0.1 & 0.1 & 0.05 & 0.1 & 0.1 & 0.05 & 0.01 & 0.05 & 0.05 & 0.1 & 0.1 & 0.1 & 0.1 \\
    $\beta$ & --- & --- & --- & 0.05 & 0.05 & 0.05 & --- & --- & --- & 0.05 & 0.05 & 0.05 & 0.01 \\
    Number of epochs & 100 & 100 & 100 & 100 & 100 & 100 & 100 & 100 & 400 & 100 & 100 & 100 & 100 \\
    \bottomrule
    \end{tabular}
}
\end{table}

\section{Synthetic Data}

\subsection{Simple Synthetic Data with Continuous Outcome}
\label{sec:continuous synthetic data}

The process iteratively generates variables $W$, $L_t$, $A_t$, and $Y_t$ over time steps $t$, for $t = 0 ,\ldots, \tau-1$. $W\sim \text{Normal}(0,1)$. At $t=0$, $L_0 \sim \text{Normal}(0.1W, 1)$, $A_0 \sim \text{Ber}( \sigma(-0.5W + L_0))$, $Y_0 \sim \text{Ber}(\sigma(-3 + 0.2W + 0.2L_0 - 2 A_0))$. For $t > 0$, $L_t \sim \text{Normal}(0.1W - 0.1A_{t-1}, 1)$, $A_t \sim \text{Ber}(\sigma(-0.5 + 0.3 W + 0.3 L_t + 2 A_{t-1}))$, $Y_t = \sigma(-3 + 0.2 W + 0.2 L_t - 2 A_t)$, $\sigma(x)=(1+e^{-x})^{-1}$ is the sigmoid function. We set the counterfactual treatment at all time-points to 1 and and evaluated the counterfactual mean of survival under this treatment policy.

\subsection{Simple Synthetic Data with Survival Outcome}
\label{sec:simple synthetic data}

The process iteratively generates variables $W$, $L_t$, $A_t$, and $Y_t$ over time steps $t$, for $t = 0 ,\ldots, \tau-1$. $W\sim \text{Normal}(0,1)$. At $t=0$, $L_0 \sim \text{Normal}(0.1W, 1)$, $A_0 \sim \text{Ber}( \sigma(-0.5W + L_0))$, $Y_0 \sim \text{Ber}(\sigma(-3 + 0.2W + 0.2L_0 - 2 A_0))$. For $t > 0$, $L_t \sim \text{Normal}(0.1W - 0.1A_{t-1}, 1)$, $A_t \sim \text{Ber}(\sigma(-0.5 + 0.3 W + 0.3 L_t + 2 A_{t-1}))$, $Y_t \sim \text{Ber}(\sigma(-3 + 0.2 W + 0.2 L_t - 2 A_t))$, with $Y_{t-1} = 1$ implying $Y_t = 1$. Here $\sigma(x)=(1+e^{-x})^{-1}$ is the sigmoid function. We set the counterfactual treatment at all time-points to 1 and and evaluated the counterfactual mean of survival under this treatment policy.

\subsection{Complex Synthetic Data with Survival Outcome}
\label{sec:complex synthetic data}

First draw parameters $\alpha_i,\beta_i\sim\mathrm{Normal}((i+1)^{-1},0.02)$ and $\gamma_i\sim 2 * \mathrm{Binom}(0.5) - 1$ for $i\in[ht]$, where $h$ is the length of time-dependency with $h=1$ corresponding to Markovian process. Then, draw error in time-dependet variables $\varepsilon^\ell_{tj}\sim\mathrm{Normal}(0,0.1)$ for $t\in[\tau]$ and $j\in[p]$, errors in treatment $\varepsilon^A_{t1}\sim\mathrm{Normal}(0,0.2)$, $\varepsilon^A_{t2}\sim\mathrm{Normal}(0,0.05)$ for $t\in[\tau]$. For each $t\in[\tau]$, $L_t=\tanh\big(\sum_{k\in[ht]}\alpha_k L_{t-k}+\beta_k\gamma_k(2A_{t-k}-1)\big) + \varepsilon^\ell_{tj}$, then draw $A_t$ from an indicator function $\mathbbm{1}\{(\sigma(s_t+)+\varepsilon^A_{t2})>0.5\}$, with $s_t=\tan\big(\prod_{j\in[p]}\bar{L}_j
 + \bar{A}\big) + \varepsilon^A_{t1}$. The outcome $Y_t$ is drawn from a Bernoulli distribution of a probability $\sigma(p_t)$ with $p_t=\tan\big(\prod_{j\in[p]}\bar{L}_j
 - 0.7 * (\bar{A} - 0.5)) - 4.5$. $Y_t=1$ if $Y_{t-1}=1$ for $t>0$. We set the counterfactual treatment policy as $\mathbbm{1}\{\sigma(s_t)>0.5\}$ for $t\in[\tau]$ and evaluated the counterfactual mean of survival under this policy.

\section{Results with Simple Synthetic Data with Survival Outcome}

\label{sec: results with simple synthetic survival data}

Results of an experiment with the simple synthetic data described in Section \ref{sec:simple synthetic data} was shown in Table \ref{table1}. Although LTMLE's strong performance on simple synthetic data is anticipated due to reduced burden in estimating nuisance parameters from Markovian dependencies, Deep LTMLE remains highly competitive in this context, equalling LTMLE's performance. Our two targeting approaches demonstrated better bias variance trade off for the estimation of the target parameter compared to the untargeted approach. Both bias and standard deviation get improved a lot for all $\tau$'s considered. The targeting step made a marked difference in terms of coverage probability, getting much closer to a nominal 95\% coverage probability compared to the one without targeting.

\begin{table*}[ht]
\centering
\caption{Results from simple synthetic data}
\label{table1}
\resizebox{\textwidth}{!}{
    \begin{tabular}{lrrrrrrcccccc}
    \toprule
     & \multicolumn{3}{c}{Bias} & \multicolumn{3}{c}{RMSE} & \multicolumn{3}{c}{Coverage} & \multicolumn{3}{c}{Mean $\hat{\sigma}_n$} \\
    Model & $\tau = 10$ & $\tau = 20$ & $\tau = 30$ & $\tau = 10$ & $\tau = 20$ & $\tau = 30$ & $\tau = 10$ & $\tau = 20$ & $\tau = 30$ & $\tau = 10$ & $\tau = 20$ & $\tau = 30$ \\
    \midrule
    LTMLE (GLM) & 0.0052 & 0.0045 & 0.0021 & 0.0202 & 0.0268 & 0.0308 & 0.95 & 0.94 & 0.93 & 0.02 & 0.03 & 0.03 \\
    LTMLE (SL) & 0.0056 & 0.0058 & 0.0061 & 0.0203 & 0.0263 & 0.0311 & 0.91 & 0.93 & 0.91 & 0.02 & 0.02 & 0.03 \\
    DeepACE & 0.0260 & -0.0061 & -0.0976 & 0.1558 & 0.0380 & 0.1232 & 0.46 & 0.93 & 0.30 & 0.06 & 0.04 & 0.04 \\
    \midrule
    Deep LTMLE & 0.0045 & 0.0176 & 0.0257 & 0.0280 & 0.0328 & 0.0710 & 0.82 & 0.87 & 0.73 & 0.02 & 0.03 & 0.04 \\
    Deep LTMLE$\dagger$ & 0.0065 & 0.0068 & 0.0009 & 0.0225 & 0.0295 & 0.0347 & 0.91 & 0.91 & 0.92 & 0.02 & 0.03 & 0.04 \\
    Deep LTMLE$\star$ & 0.0059 & 0.0106 & 0.0036 & 0.0226 & 0.0296 & 0.0360 & 0.91 & 0.92 & 0.90 & 0.02 & 0.03 & 0.04 \\
        \bottomrule
    \end{tabular}
}
\end{table*}

\subsection{Semi-Synthetic Data}

\label{sec:semi-synth}

As a compromise, we conducted several additional experiments with semi-synthetic data from the real world data as used in previous studies \cite{Bica2020Estimating,Frauen_Hatt_Melnychuk_Feuerriegel_2023}. For this experiment, we used covariates from the Circulatory Risk in Communities Study (CIRCS) and fit outcome regression given the history through each time point using XGBoost with early stopping. Outcomes were then generated using this fitted regression model. For the experiment, we sample 1000 observations from the empirical dstribution of covariates $W,L_t,A_t$ and generate $Y_t$ for $t=1,\ldots,\tau$ with $\tau\in\{10,20.30\}$.

\section{Extension to Survival Analysis with Censoring}

\label{sec:extended-model}

In this section, we describe the extended LTMLE model with censoring for the real world application in Section \ref{sec:rwd}. We assume the following order of observed nodes $O=(W=W_0, L_1, A_1, C_1, Y_1 , L_2 , A_2 , C_2 , Y_2 , \ldots , L_\tau , A_\tau , C_\tau , Y_\tau=Y)$, where $C_t$ are binary censoring nodes with $C_t=1$ indicating one being censord. Our interest is to estimate the risk of our outcome $Y_\tau$, the mortality of the individual. However, our observation period spans long-term, individuals are at risk of being censored. Censoring $C_t$ is loss of follow-up from administrative reasons, for example, move to other areas or denial of participation in the survey. We assume degenerations of nodes. When we observe a jump in $Y$ or $C$ nodes, the process halts and all nodes after the jump remain constant with the last observed values. For example, if $Y_t=1$, then $Y_s=1$, $C_s=0$, $A_{s-1}=A_{t-1}$, and $L_s=L_t$ for all $s > t$.

We constructed a Deep LTMLE similar to the one describe in Section \ref{sec:Deep LTMLE} with this structure. The only difference is an additional component of censoring mechanism $\lambda^c$ which is involved in the clever covariate $I_t$ and the loss function:
\begin{align}
    I_t(\pi, \lambda^c) & = \prod_{s=1}^t \frac{d(g\otimes\mathbbm{1}(C_s=0))}{d(\pi_t\otimes \lambda^c_t)}(O), \text{ and }\\
    \mathcal{L}(Q, V, \pi, \lambda^c) &= \mathcal{L}^Q(Q, V) + \alpha\mathcal{L}^e(\pi, A) + \beta\mathcal{L}^c(\lambda^c, C),
\end{align}
where $\beta$ is an additional scaling hyperparameter. The counterfactual treatment on the censoring process is $\mathbbm{1}(C_t=0)$ meaning supression of censoring. Estimates of the target parameter and the efficient influence curve for different treatment strategies are computed using Deep LTMLE, and average treatment effects (ATEs) and its EIC were computed using the delta method. Based on the estimated EICs of the target parameters at each time point t, we constructed a simultaneous confidence intervals.

\end{document}